\documentclass{article} 
\usepackage{natbib} 
\usepackage{authblk} 

\usepackage{amsmath,amsfonts,bm}

\def\eqref#1{equation~\ref{#1}}

\def\1{\bm{1}}

\DeclareMathAlphabet{\mathsfit}{\encodingdefault}{\sfdefault}{m}{sl}
\SetMathAlphabet{\mathsfit}{bold}{\encodingdefault}{\sfdefault}{bx}{n}

\newcommand{\E}{\mathbb{E}}

\newcommand{\R}{\mathbb{R}}

\usepackage{hyperref}
\usepackage{url}
\usepackage{algorithm}
\usepackage{algpseudocode}
\usepackage{graphicx}
\usepackage{caption}
\usepackage{subfig}
\usepackage{amssymb}
\usepackage{docmute}

\usepackage{wrapfig} 

\usepackage{booktabs} 
\usepackage{adjustbox} 

\usepackage[english]{babel}
\usepackage{amsthm, amsmath}
\usepackage[dvipsnames]{xcolor}

\newtheorem{lemma}{Lemma}
\newtheorem{definition}{Definition}
\newtheorem{remark}{Remark}
\newtheorem{theorem}{Theorem}
\newtheorem{corollary}{Corollary}

\newcommand{\rebuttal}[1]{#1}

\title{Evolving Domain Generalization}

\author[1]{William Wei Wang\thanks{Corresponde to William Wei Wang wwang828@uwo.ca}}
\author[1]{Gezheng Xu}
\author[1]{Ruizhi Pu}
\author[1]{Jiaqi Li}
\author[2]{Fan Zhou}
\author[2]{Changjian Shui}
\author[1]{Charles Ling}
\author[2, 3]{Christian Gagn\'e}
\author[1, 4]{Boyu Wang}
\affil[1] {Western University}
\affil[2] {Universit\'e Laval}
\affil[3] {Mila, Canada CIFAR AI Chair}
\affil[4] {Vector Institute}

\begin{document}

\maketitle

\begin{abstract}
Domain generalization aims to learn a predictive model from multiple different but related source tasks that can generalize well to a target task without the need of accessing any target data. Existing domain generalization methods ignore the relation between tasks, implicitly assuming that all the tasks are sampled from a stationary environment. Therefore, they can fail when deployed in an evolving environment. To this end, we formulate and study the \emph{evolving domain generalization} (EDG) scenario, which exploits not only the source data but also their evolving pattern to generate a model for the unseen task. Our theoretical result reveals the benefits of modeling the relation between two consecutive tasks by learning a globally consistent directional mapping function. In practice, our analysis also suggest solving the EDG problem in a meta-learning manner, which leads to \emph{directional prototypical network}, the first method for the EDG problem. Empirical evaluation on both synthetic and real-world data sets validates the effectiveness of our approach.
\end{abstract}

\section{Introduction}

Modern machine learning techniques have achieved unprecedented success over the past decades in various areas.
However, one fundamental limitation of most existing techniques is that a model trained on one data set cannot generalize well on another data set if it is sampled from a different distribution. Domain generalization (DG) aims to alleviate the prediction gap between the observed source domains and an \emph{unseen} target domain by leveraging the knowledge extracted from multiple source tasks~\citep{wjd_77,wjd_88,wjd_12}. 


Existing DG methods can be roughly categorized into three groups: data augmentation/generation, disentangled/domain-invariant feature learning, and meta-learning \citep{jd_survey}. One intrinsic problem with these methods is that they treat all the domains equally and ignore the relationship between them, implicitly assuming that they are all sampled from a stationary environment. However, in many real-world applications, the data are usually collected sequentially and the learning tasks can vary in an \emph{evolving} manner. For example, geological exploration is often carried out periodically and the distribution of data collected can change from year to year due to environmental changes. Medical data are also often collected with age or other indicators as intervals, and there is an evolving trend in the data of different groups. As a more concrete example,  Fig.~1(a) shows several instances from the rotated MNIST (RMNIST) data set, a widely used benchmark in the DG literature, where the digit images of each subsequent domain are rotated by $15^\circ$. Fig.~1(b) reports the generalization performances of several state-of-the-art DG algorithms on the data set, from which it can be clearly observed that the performances drop when deploying the models on outer domains (i.e., tasks of 0 and 75 degrees). The results indicate that the algorithms ignore the evolving pattern between the domains properly. As a consequence, they are good at ``interpolation" but not at ``extrapolation". 

In this paper, we formulate this learning scenario as \emph{evolving domain generalization} (EDG), which aims to capture and exploit the evolving patterns in the environment. In contrast to most existing DG methods, which produce models that are {isotropic} with respect to all the domains, EDG can generalize to a target domain along a specific direction by extracting and leveraging the relations between source tasks. Specifically, we develop a novel theoretical analysis that highlights the importance of modeling the relation between two consecutive tasks to extract the evolving pattern of the environment. Moreover, our analysis also suggests learning a globally consistent directional mapping function via meta-learning. Inspired the theoretical results, we slightly modify prototypical networks \citep{proto} and propose \emph{directional prototypical networks} (DPNets), a simple and efficient EDG algorithm that adapts to the environment shift and generalizes well on the evolving target domain. As a comparison, Fig. 1(c) shows the performance improvement of DPNets over the other algorithms on RMNIST data set. It can be observed that the performance gap between DPNets and the other baseline algorithms has widened as the domain distance increases. More details can be found in Section~\ref{sec:exp}.

Here, we would like to emphasize the key difference between EDG and domain adaptation in evolving domains \citep{laed,cida,cudaal}. While both learning paradigms aim to tackle the issue of evolving domain shifts, the latter still requires unlabeled instances from the target domain. In this sense, EDG is  more challenging  and existing theoretical and algorithmic results cannot be applied to this problem directly.


\begin{figure*}[!tb]
		\centering 
		\subfloat[]{\includegraphics[width=0.33\textwidth]{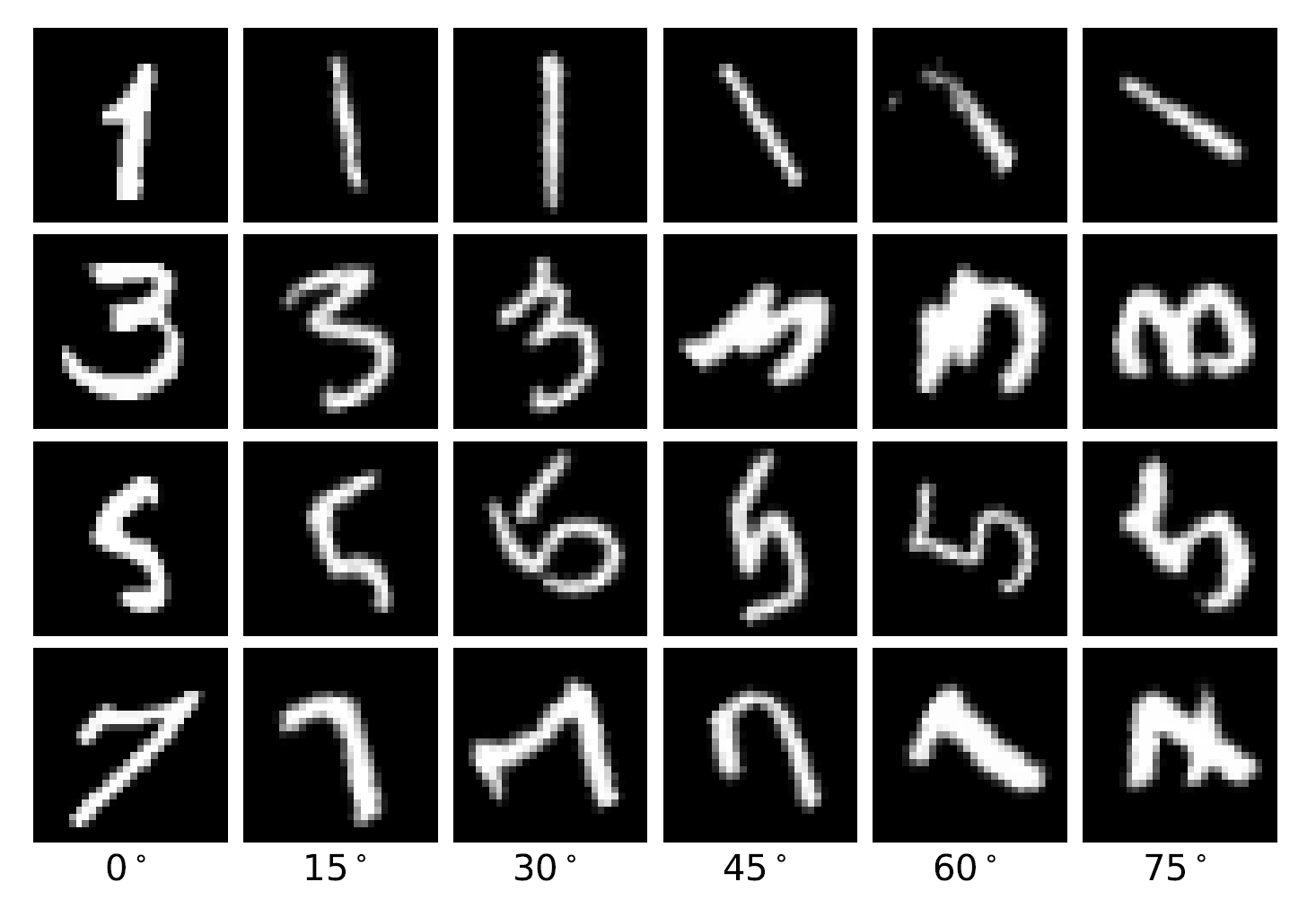}\label{motivation_figure:rmnist}}
		\subfloat[]{\includegraphics[width=0.33\textwidth]{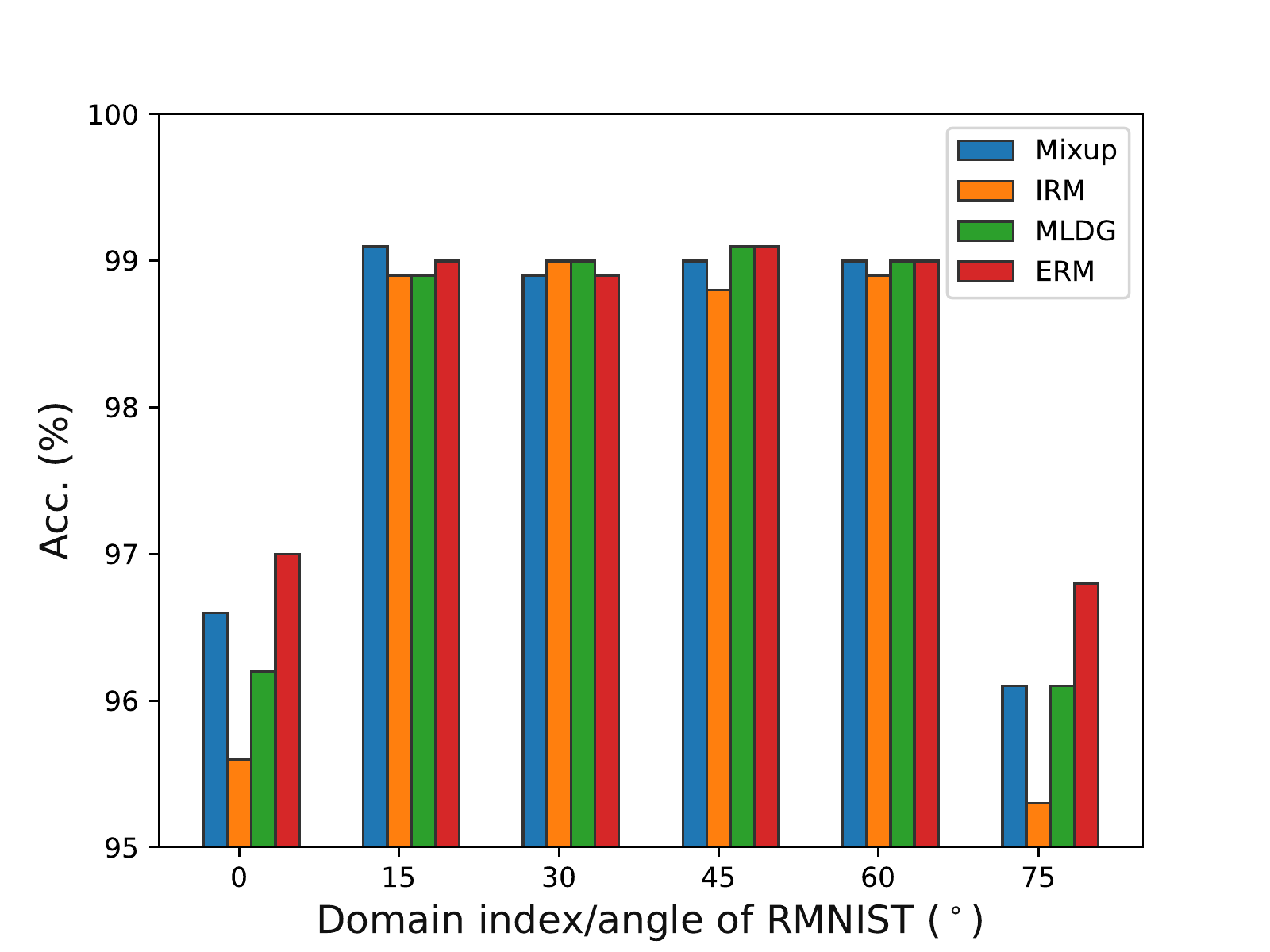}\label{motivation_figure:acc_baselines}}
		\subfloat[]{\includegraphics[width=0.33\textwidth]{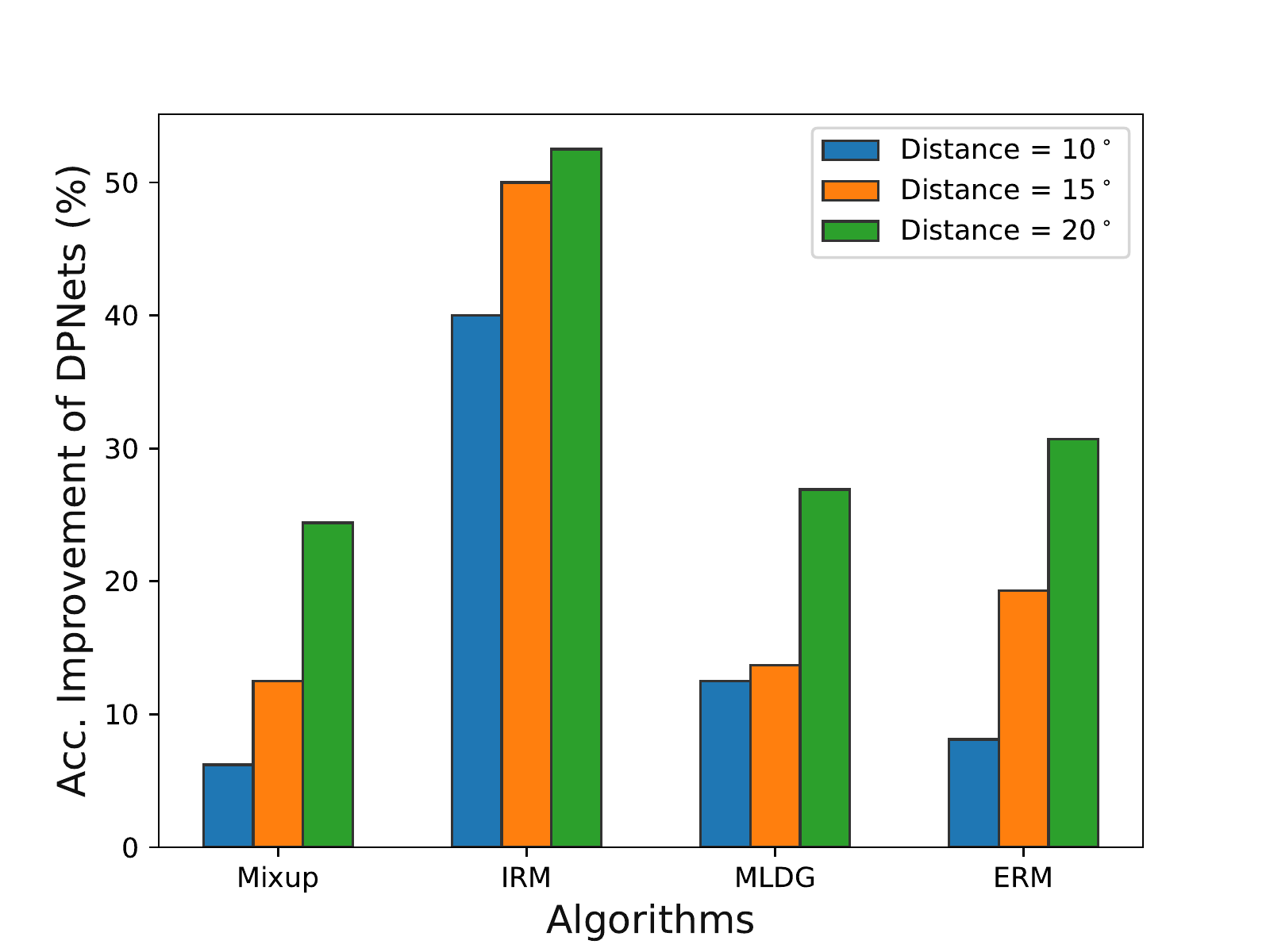}\label{motivation_figure:acc_forward_backward}}
\caption{(a) Evolving manner among RMNIST domains. (b) Accuracy of traditional DG methods on evolving domains. These methods cannot generalize well on outer domains ($0^\circ$ and $75^\circ$). (c) Comparison between the performance of our method and baselines on outer domains. The proposed method outperforms all the baselines.}
\label{motivation_figure}
\end{figure*}


\section{Related Work}

\paragraph{Domain Generalization (DG).} Domain generalization aims to train a model which generalizes on all domains. Existing DG methods can be classified into three categories. The first and most popular category is representation learning, which focus on learning a common representation across domains. It can be achieved by domain-invariant representation learning \citep{wjd_24,wjd_59,wjd_77,wjd_88} and feature disentanglement \citep{wjd_93, wjd_101}. The former focuses on aligning latent features across domains, and the later tries to distill domain-shared features. Secondly, data manipulation can also empower the model with generalization capability. Data manipulating techniques include data augmentation \citep{wjd_28, wjd_35}, which usually extend the dataset by applying specific transformations on existing samples, and data generation \citep{wjd_38}, which often applies neural networks to generate new samples. \rebuttal{\cite{dirt} convert DG to an infinite-dimensional constrained statistical learning problem under a natural model of data generation. The theoretically grounded method proposed in \cite{dirt} leverages generating model among domains to learn domain-invariant representation.} The last part are meta-learning. As a widely applicable method, meta-learning framework \citep{wjd_12, wjd_13, wjd_14} is used to improve the generalizing capability by simulating the shift among domains. 
\rebuttal{Apart from the above three categories, there are some other learning strategies that help to improve the generalization ability of the model. \cite{wjd_104} tries to ensemble multiple models into a unified one which can generalize across domains. The DRO-based methods \citep{rahimian2019distributionally} which aim to learn a model at worst-case distribution scenario also match the target of DG well. Besides, gradient operation \citep{wjd_112}, self-supervision \citep{wjd_114} and random forest \citep{wjd_115} are also exploited to improve generalizing capability.}
Different from the existing DG methods that focus on learning one unified model for all domains, our approach tries to train prediction model for the target domain specifically by leveraging the evolving pattern among domains.

\paragraph{Evolving Domain Adaptation (EDA).} Many previous work in domain adaptation area notice the evolving pattern of domains and leverage it to improve the performance in different settings. \cite{laed} proposes a meta-adaptation framework which enables the learner to adapt from one single source domain to continually evolving target domains without forgetting. \cite{cudaal} tries to adapt to multiple target domains sequentially without forgetting, while the most important difference is that there is no assumption about the evolving pattern between these domains. \cite{usgda} focuses on adapting from source domain to the target domain with large shifts by leveraging the unlabeled intermediate samples. \cite{cida} combines the traditional adversarial adaptation strategy with a novel regression discriminator that models the encoding-conditioned domain index distribution. \cite{chen2021gradual} investigate how to discover the sequence of intermediate domains without index information then adapt to the final target. The experimental results and theoretical analysis demonstrate the value of leveraging index information when working on evolving domains. \rebuttal{These studies fully demonstrate that leveraging evolving pattern between domains is beneficial and worth more exploration. However, there are two significant limitation in previous works. The first one is the requirements for accessing unlabeled data in the target domain. The second one is that all previous theoretical results are base on the assumption that distance between sequential domains is small, which would become vacuous as the environment evolves. This is contrary to the fact that more domains provide more evolving information which can help to improve the performance. In this paper, the evolving information is leveraged without accessing any samples from the target domain. Also, our theoretical results are based on proposed $\lambda$-consistency, a intuitive and realistic measurement of evolving level in the environments.}

\section{Theoretical Analysis and methodology}
Let $\{\mathcal D_1(x,y),\mathcal D_2(x,y),...,\mathcal D_m(x,y)\}$ be $m$ observed source domains sampled from an environment $\mathcal{E}$ 
where $x\in \mathcal{X}$ and $y \in \mathcal{Y}$ are, respectively, the data point and its corresponding label. The goal of DG is to learn a hypothesis $h\in \mathcal{H}$ so that it can have a low risk on an unseen but related target domain $\mathcal{D}_t$:
\begin{align*}
    R_{\mathcal{D}_t}(h) \triangleq \mathbb E_{(x,y)\sim \mathcal D_t}[\ell(h(x),y)]
\end{align*}



where $\ell:\mathcal Y\times\mathcal Y\rightarrow \mathbb R_+$ is a non-negative loss function, and $\mathcal{H}$ is a hypothesis class that maps $\mathcal X$ to the set $\mathcal Y$. In the setting of traditional DG, as there is no relation exploited between the source domains and the target domain, most existing techniques essentially either ``enlarge" the input space  $\mathcal{X}$ along all possible directions~\citep{volpi2018generalizing,shankar2018generalizing,qiao2020learning} or learn a domain-invariant feature representation via domain alignment~\citep{li2018domain,wjd_88,zhao2020domain}. In contrast, the objective of EDG is to generalize the model on $\mathcal{D}_t$ along a specific direction when there is underlying evolving pattern between the source domains and $\mathcal{D}_t=\mathcal{D}_{m+1}$.





\subsection{Theoretical Motivations}

  In order to leverage the evolving pattern in $\mathcal{E}$, it is reasonable to assume that such a pattern can be captured by a globally consistent mapping function $g \in \mathcal{G}: \mathcal{D}_{i+1}^g \triangleq g(\mathcal{D}_i)$ in a way such that the \emph{synthetic} domain $\mathcal{D}_{i+1}^g$ is close to $\mathcal{D}_{i+1}$ as much as possible, where $\mathcal{G}$ is the class of mapping functions. We first obtain the following bound of the risk on the target domain $\mathcal{D}_t$ with respect to $\mathcal{D}_t^g$, as shown in Lemma~\ref{lemmastart}.


\begin{lemma}
\label{lemmastart}
Let ${\mathcal D^g_t}(h) = g(\mathcal{D}_m)$ be the \emph{synthetic} target domain, and suppose the loss function $\ell$  is bounded within an interval $G:G=\max(\ell)-\min(\ell)$. Then, for any $h \in \mathcal{H}$, its target risk $R_{\mathcal{D}_t}(h)$ can be upper bounded by:
\begin{align*}
    R_{\mathcal D_t}(h)\leq R_{\mathcal D^g_t}(h) +\frac{G}{\sqrt{2}}\sqrt{d_{JS}(\mathcal D^g_t||\mathcal D_t)},
\end{align*}
where $d_{JS}(\mathcal D^g_t||\mathcal D_t)$ is the Jensen-Shannon (JS) divergence between $\mathcal D^g_t$ and $\mathcal D_t$~\citep{jsd}.
\end{lemma}
\begin{remark}
\emph{To achieve a low risk on $\mathcal{D}_t$, Lemma~\ref{lemmastart} suggests (1) learning $h$ and $g$ to minimize the risk over the synthetic domain $\mathcal{D}_t^g$ and (2) learning $g$ to minimize the JS divergence between $\mathcal{D}_t^g$  and $\mathcal{D}_t$. While in practice $R_{\mathcal D^g_t}(h)$ can be approximated by the empirical risk, Lemma~\ref{lemmastart}  still cannot provide any practical guidelines for learning $g$ since $\mathcal{D}_t$ is unavailable. Moreover, note that $\mathcal{D}_t^g$  can be replaced by $g(\mathcal{D}_i)$ for any other source domain $i$ in $\mathcal{E}$ and the bound still holds. Thus, Lemma~\ref{lemmastart} does not provide any theoretical insight into how to discover and leverage the evolving pattern in $\mathcal{E}$. }
\end{remark}

Intuitively, capturing the evolving pattern in $\mathcal{E}$ is hopeless if it varies arbitrarily. On the other hand, if the underlying pattern is consistent over domains, it is reasonable to assume that there exists $g^*\in \mathcal G$ would perform consistently well over all the domain pairs. For example, given numbers $100, 202, 301$, one would expect that the numbers increase by around 100 and the next number will be around $400$, but it is challenging to guess the number if the first three numbers are $-20, 1300, 4$.  To formulate the this intuition, we  first introduce the notion of \emph{consistency} of an environment $\mathcal{E}$ below.

\begin{definition}[Consistency] 
\label{def1}
Let $g^* = \arg\min_{g\in \mathcal{G}} \max_{\mathcal{D}_i\in\mathcal{E}} d_{JS}(\mathcal{D}_i^g||\mathcal{D}_i)$  be the ideal mapping function in the worst-case domain. Then, an evolving environment $\mathcal{E}$ is \emph{$\lambda$-consistent} if the following holds:
\begin{align*}
    |d_{JS}(\mathcal D_i^{g^*}||\mathcal D_{i})-d_{JS}(\mathcal D_j^{g^*}||\mathcal D_{j})| \le \lambda, \qquad \forall \mathcal{D}_i, \mathcal{D}_j \in \mathcal{E}.
\end{align*}
\end{definition}

Note that $\lambda$ does not characterize how fast $\mathcal{E}$ evolves, but if there exists a global mapping function that can model the evolving pattern consistently well in $\mathcal{E}$. 

Given the definition of consistency, we can bound the target risk in terms of $d_{JS}(\mathcal D^g_i||\mathcal D_i)$ in the source domains.

\begin{theorem}
\label{theoremds}
Let $\{\mathcal D_1,\mathcal D_2,...,\mathcal D_m\}$ be $m$ observed source domains sampled sequentially from an evolving environment $\mathcal{E}$, and $\mathcal{D}_t$ be the next unseen target domain: $\mathcal{D}_t = \mathcal{D}_{m+1}$. Then, if $\mathcal{E}$ is $\lambda$-consistent, we have
\begin{align*}
   R_{\mathcal{D}_t}(h) \le R_{\mathcal{D}_t^{g^*}}(h)+\frac{G}{\sqrt{2(m-1)}}\Bigg(\sqrt{\sum _{i=2}^{m}d_{JS}(\mathcal D_{i}^{g^*}||\mathcal D_{i})}+\sqrt{(m-1)\lambda}\Bigg).
\end{align*}
\end{theorem}
\begin{remark}\label{rm2}
\emph{{\bf (1)}~Theorem~\ref{theoremds} highlights the role of the mapping function and $\lambda$-consistency in EDG. Given $g^*$, the target risk $R_{\mathcal{D}_t}(h)$ can be upper bounded by in terms of loss on the synthetic target domain $R_{\mathcal{D}_t^{g^*}}(h)$, $\lambda$, and the JS divergence between $\mathcal{D}_i$ and $\mathcal{D}_i^{g^*}$ in all \emph{observed} source domains. When $g^*$ can properly capture the evolving pattern of $\mathcal{E}$, we can train the classifier $h$ over the synthetic domain $\mathcal{D}_t^{g^*}$ generated from $\mathcal{D}_m$ and can still expect a low risk on $\mathcal{D}_t$. {\bf(2)}~$\lambda$ is \emph{unobservable} and is determined by $\mathcal{E}$ and $\mathcal{G}$. Intuitively, a small $\lambda$ suggests high \emph{predictability} of $\mathcal{E}$, which indicates that it is easier to predict the target domain $\mathcal{D}_t$. On the other hand, a large $\lambda$ indicates that there does not exist a global mapping function that captures the evolving pattern consistently well over domains. Consequently, generalization to the target domain could be challenging and we cannot expect to learn a good hypothesis $h$ on $\mathcal{D}_t$. {\bf (3)}~In practice, $g^*$ is not given, but can be learned by minimizing $d_{JS}(\mathcal D_{i}^{g}||\mathcal D_{i})$ in source domains. Besides, aligning $D_{i}^{g}$ and $\mathcal{D}_i$ is usually achieved by} representation learning: \emph{that is, learning $g: \mathcal{X} \rightarrow \mathcal{Z}$ to minimize $d_{JS}(\mathcal D_{i}^{g}||\mathcal D_{i}), \forall z \in \mathcal{Z}$.}
\end{remark}
In addition, we can decompose $\mathcal{D}(x,y)$ into marginal and conditional distributions to motivate more practical EDG algorithms. For example, when it is decomposed into 
class prior $\mathcal{D}(y)$ and semantic  conditional distribution $\mathcal{D}(x|y)$, we have the following Corollary.





\begin{corollary}
\label{corollay1}
Following the assumptions of Theorem 1, the target risk can be bounded by
{\small
\begin{align*}
 &R_{\mathcal D_t}(h) \leq  R_{\mathcal D_t^{g^*}}(h)+\frac{G}{\sqrt{2(m-1)}}\Bigg(\underbrace{\sqrt{\sum _{i=2}^{m}d_{JS}(\mathcal D^{g^*}_{i}(y)||\mathcal D_{i}(y))}}_{\bold{I}} + \sqrt{(m-1)\lambda}\\
& \hspace{24pt}+\underbrace{\sqrt{\sum _{i=2}^{m}\mathbb E_{y\sim \mathcal D_{i}^{g^*}(y)}d_{JS}(\mathcal D^{g^*}_{i}(x|y)||\mathcal D_{i}(x|y))}}_{\bold{II}}+\underbrace{\sqrt{\sum _{i=2}^{m}\mathbb E_{y\sim \mathcal D_{i}(y)}d_{JS}(\mathcal D^{g^*}_{i}(x|y)||\mathcal D_{i}(x|y))}}_{\bold{III}} \Bigg).
\end{align*}
}
\end{corollary}
\begin{remark}\label{rmk3}
\emph{To generalize well to $\mathcal{D}_t$, Corollary~\ref{corollay1} suggests that a good mapping function should capture both label shifts (term I) and semantic conditional distribution shifts (terms II \& III). If we further assume that  the label distribution does not evolve over domains\footnote{When label shifts exist, term I can be minimized by reweighting/resampling the instances according to the class ratios between domains~\citep{pmlr-v139-shui21a}.}, we will have I = 0 and II = III, and the upper bound can be simplified as}
{\small
\begin{align*} 
 R_{\mathcal D_t}(h) &\leq  R_{\mathcal D_t^{g^*}}(h)+\frac{G}{\sqrt{2(m-1)}}\Bigg(2{\sqrt{\sum _{i=2}^{m}\mathbb E_{y\sim \mathcal D_{i}(y)}d_{JS}(\mathcal D^{g^*}_{i}(x|y)||\mathcal D_{i}(x|y))}} + \sqrt{(m-1)\lambda}\Bigg).
\end{align*}
}
\end{remark}
Finally, we note two key theoretical differences between EDG and previous studies of DA  in evolving environments \citep{david,laed}. (1)~In DA, the target risk is bounded in terms of  source and  target domains (e.g., $\mathcal{H}\Delta \mathcal{H}$-divergence), while  our analysis relies on the distance between synthetic and  real domains. (2)~DA theories are built upon the assumption that there exists an ideal joint hypothesis that achieves a low combined error on both domains, while our assumption is the $\lambda$-consistency of $\mathcal{E}$. These differences eventually lead to fundamentally different guidelines for EDG, as shown in Section~\ref{sec:algo}.

\subsection{Practical Implementations}\label{sec:algo}

Our analysis reveals several general strategies to follow when designing an algorithm for EDG. 

\begin{itemize}
    \item[(i)] Learning the mapping function $g$ to capture the evolving pattern by minimizing  the distance between the distributions of synthetic and real domains.  
    \item[(ii)] Learning $g$ and $h$ to minimize the  risk on the synthetic target domain $\mathcal{D}_t^{g}$.
    \item[(iii)]  Note that $\mathcal D^{g}_{i+1} = g(\mathcal{D}_i)$ is produced from $\mathcal{D}_i$, but its quality is evaluated on $\mathcal{D}_{i+1}$. Thus, minimizing $d_{JS}(\mathcal D^{g}_{i}||\mathcal D_{i})$ naturally suggests a {\bf meta-learning} strategy for learning $g$.
\end{itemize}

\rebuttal{In practice, mapping the samples from $\mathcal{D}_{i}$ to $\mathcal{D}_{i+1}$ is not necessarily performed in the original data space since our  ultimate goal is making predictions rather than generating instances themselves. Thus, we minimize the distance between $\mathcal{D}_{i+1}^g$ and $\mathcal{D}_{i+1}$ in a representation space.} 
Based on these ideas, we slightly modify prototypical networks \citep{proto} and propose \emph{directional prototypical networks} (DPNets) for EDG. \rebuttal{Specifically, the mapping function $g$ of DPNets 
consists of two different embedding functions: $f_\phi$ for $\mathcal{D}_i$ and $f_\psi$ for $\mathcal{D}_{i+1}$, where $\phi$ and $\psi$ are learnable parameters. The key idea of DPNets is to learn $g=\{f_\phi, f_\psi\}$ to capture the evolving pattern of $\mathcal{E}$ by estimating the prototypes of $f_\psi(\mathcal{D}_{i+1})$ using $f_\phi(\mathcal{D}_i)$. As each prototype can be viewed as the centroid of instances of each class, which is an approximation of the semantic conditional distribution of each class~\citep{xie2018learning,pmlr-v139-shui21a},
DPNets essentially minimizes the distance between $\mathcal{D}^{g}_{i+1}(x|y)$ and $\mathcal{D}_{i+1}(x|y)$, as suggested  by Corollary~\ref{corollay1} Remark~\ref{rmk3}. }  

Let $\mathcal{S}_i = \{(x_n^i,y_n^i)\}_{n=1}^{N_i}$ be the data set of size $N_i$ sampled from $\mathcal{D}_i$, and $\mathcal{S}_i^k$ be the subset of $\mathcal S_i$ with class $k\in\{1,...,K\}$, where $K$ is the total number of classes. Then, the prototype of domain $i$ is the mean vector of the \emph{support} instances belonging to $\mathcal{S}_i^k$:
\begin{align*}
    c_i^k=\frac{1}{|\mathcal{S}^k_i|}\sum_{(x_n^i,y_n^i)\in \mathcal{S}^k_i}f_\phi(x_n^i)
\end{align*}

In~\citep{proto}, the prototype is used to produce a distribution $\mathcal{D}({y=k}|x)$ to make a prediction for a \emph{query} instance $x$ in the context of few-shot learning, and the support and query instances are sampled from the same domain. By contrast, in DPNets, the prototypes are computed from the support set $\mathcal{S}_i$ through the embedding function $f_\phi$, but the query instances are from $\mathcal{S}_{i+1}$ and are passed through another function $f_\psi$. Then, the predictive distribution for $\mathcal{D}_{i+i}$ is given by 
\begin{align}\label{eq1}
    \mathcal{D}_{\phi,\psi}(y^{i+1}=k|x^{i+1})=\frac{\exp(-d(f_\psi(x^{i+1}), c^k_i))}{\sum_{k'=1}^{K}\exp(-d(f^{}_\psi(x^{i+1}), c^{k'}_i)},
\end{align}
where $d: \mathcal{Z}\times\mathcal{Z}\rightarrow [0,+\infty)$ is a distance function of embedding space, and we adopt squared Euclidean distance in our implementation, as suggested in~\citep{proto}. During the training stage, at each step, we randomly choose the data sets $\mathcal S_i,\mathcal S_{i+1}$ from two consecutive domains as support and query sets, respectively. Then, we sample $N_B$ samples from each class $k$ in $\mathcal S_i$, which is used to compute prototype $c_i^k$ for the query data in $\mathcal{S}_{i+1}$ .  Model optimization proceeds by minimizing  the negative log-probability: $J_{\phi,\psi}=-\log \mathcal{D}_{\phi,\psi}(y^{i+1}=k|x^{i+1})$. The pseudocode to compute $J_{\phi,\psi}$  for a training episode is shown in Algorithm~\ref{alg:cap}. In the testing stage, we pass the instances from $\mathcal S_m$ and $\mathcal S_t$ through $f_\phi$ and  $f_\psi$ respectively as support and query sets, and then make predictions for the instances in $\mathcal{D}_t$ using Eq.~(\ref{eq1}).

\begin{algorithm}[t]
\caption{The loss computation for DPNets (one episode)}\label{alg:cap}
\begin{algorithmic}

\State \textbf{Input: }  $\{\mathcal S_1,\mathcal S_2,...,\mathcal S_m\}$: $m$ data sets from consecutive domains. $N_B$: the number of support and query instances for each class. $\textsc{RandomSample}(\mathcal S, N)$: a set of $N$ instances sampled uniformly from the set $\mathcal S$ without replacement.
\State \textbf{Output: } The loss $J_{\phi,\psi}$ for a randomly generated training episode.

\State $t\leftarrow \textsc{RandomSample}(\{1,...,m\})$ 

\For{$k$ in $\{1,...,K\}$}
    \State $S^k\leftarrow \textsc{RandomSample}(\mathcal S_{i}^{k},N_B)$ 
    
    \State $Q\leftarrow \textsc{RandomSample}(\mathcal S_{i+1}^k,N_B)$ 
    
    \State $c^k=\frac{1}{|S^k|}\sum_{(x_j,y_j)\in S^k}f_\phi(x_j)$ \Comment{Compute prototypes}
\EndFor

\State $J(\phi,\psi)\leftarrow0$

\For{$k$ in $\{1,...,K\}$} 

    \For {$(x,y)$ in $Q$} 
    
    \State $J_{\phi,\psi}\leftarrow J_{\phi,\psi}+\frac{1}{KN_B}[d(f_\psi(x),c^k)+\log\sum_{k'}\exp(-d(f_\psi(x),c^k)]$
    
    \EndFor
\EndFor

\end{algorithmic}
\end{algorithm}

\section{Experiments}\label{sec:exp}


\subsection{Experimental Setup}

We evaluate our algorithm on an extensive collection of five data sets, including \textbf{two synthetic data sets} (Envolving Circle \citep{cida} and Rotated Plate  and \textbf{three real-world data sets} (RMNIST \citep{wjd_59}, Portrait \citep{usgda,chen2021gradual}, and Cover Type \citep{usgda}).


(1) \textbf{Evolving Circle (EvolCircle, Fig. \ref{fig:evolcircle})} consists of 30 evolving domains, where the instances are generated from 30 2D Gaussian distributions with the same variances but different centers uniformly distributed on a half-circle. 
(2) \textbf{Rotated Plate (RPlate, Fig.~\ref{fig:rplate})} consists of 30 domains, where the instances of each domain is generated by the same Gaussian distribution but the decision boundary rotates from $0^\circ$ to $348^\circ$ with an interval of $12^\circ$. 
(3) \textbf{Rotated MNIST (RMNIST)} We randomly select only 2400 instances in raw MNIST dataset and split them into 12 domains equally. Then we apply the rotations with degree of $\theta=\{0^\circ, 10^\circ, ..., 110^\circ\}$ on each domain respectively. The amount of samples in each domain is only 200, which makes this task more challenging.
(4) \textbf{Portrait} This task is to classify gender based on the photos of high school seniors across different decades \citep{portrait}. We divided the dataset into 12 domains by year. 
(5) \textbf{Cover Type} data set aims to predict cover type (the predominant kind of tree cover) from 54 strictly cartographic variables. 
To generate evolving domains, we sort the samples by the ascending order of the distance to the water body, as proposed in \citep{usgda}. Then we equally divided the data set into 10 domains by distance. 
(5) \textbf{FMoW} A large satellite image dataset with target detection and classification tasks \citep{christie2018functional} . We select 5 common classes to compose a classfication task. The dataset is divided into 19 domains by time. 

\begin{figure*}[tbp]
		\centering 
		\includegraphics[width=1.0\textwidth]{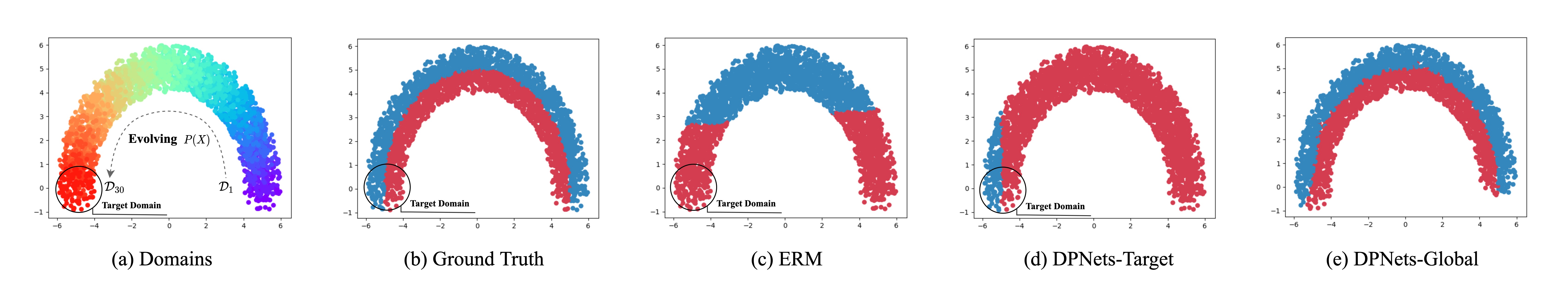}
\caption{Visualization of the EvolCircle data set. (a) 30 domains indexed by different colors, where the left bottom one is target domain.  (b) Positive and negative instances are denoted by red and blue dots respectively. \rebuttal{(c) The decision boundaries learned by ERM. (d) Decision boundaries of last model on all domains. (e) Decision boundaries of models in each domain. (the results of first domain are missing due to the lack of prototypes}.}
\label{fig:evolcircle}
\end{figure*}
\begin{figure*}[tbp]
		\centering 
		\includegraphics[width=1.0\textwidth]{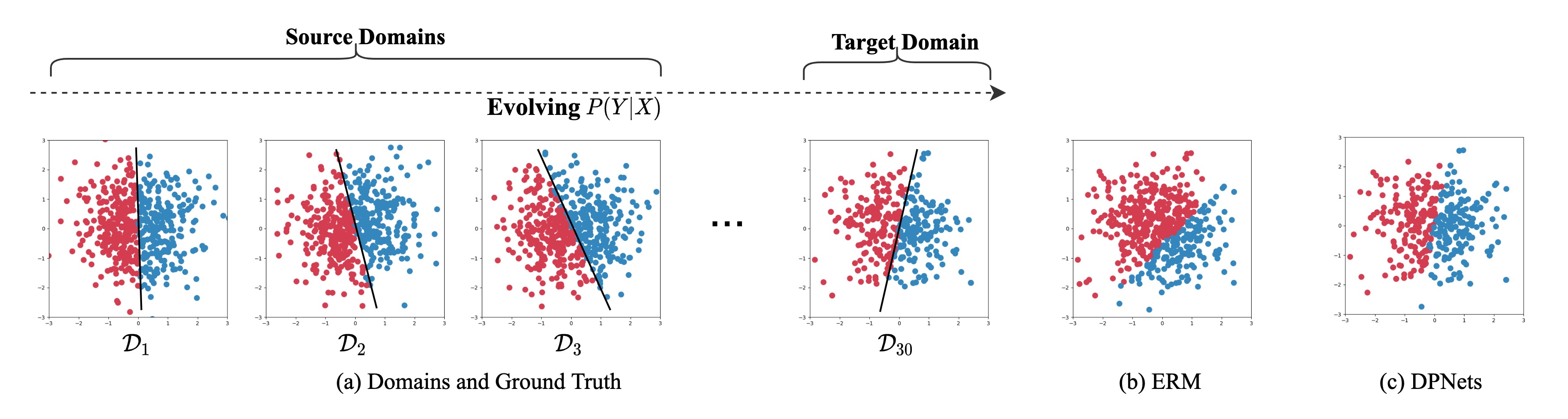}
\caption{Visualization of the RPlate data set. (a) The true decision boundaries evolves over domains. (b) \& (c) The decision boundaries learned by ERM and DPNets on the target domain. }
\label{fig:rplate}
\end{figure*}

We compared the proposed method with the following baselines: (1) \textbf{GroupDRO}~\citep{sagawa2019distributionally}; (2) \textbf{MLDG}~\citep{li2018learning}; (3) \textbf{MMD}~\citep{li2018domain}; (4) \textbf{SagNet}~\citep{nam2021reducing}; (5) \textbf{VREx}~\citep{krueger2021out}; (6) \textbf{SD}~\citep{pezeshki2020gradient}; (7) \textbf{IRM}~\citep{wjd_88}; (8) \textbf{Mixup}~\citep{yan2020improve}; (9) \textbf{CORAL}~\citep{sun2016deep}; (10) \textbf{MTL}~\citep{blanchard2021domain}; (11) \textbf{RSC}~\citep{huang2020self}; (12) \textbf{DIRL}~\citep{dirt}; (13) \textbf{Original Prototypical Network}~\citep{proto}; (14) \textbf{ERM}~\citep{vapnik1998Statistical}.
All the baselines and experiments were implemented with DomainBed package \citep{domainbed} under the same setting, which guarantees extensive and sufficient comparisons. Specifically, for each algorithm and data set, we conduct a random search (Bergstra and Bengio, 2012) of 20 trials over the hyperparameter distribution, and for each hyperparameter, five independent experiments with different random seeds are repeated to reduce the variances. 
Other details of hyperparameters and experimental setup are provided in the appendix. 



\begin{table}[!t]
\caption{Comparison of accuracy (\%) among different methods.  }
\label{tab:all_res}
\begin{center}
\adjustbox{max width=\textwidth}{%
\begin{tabular}{lccccccc}
\toprule
\textbf{Algorithm}         & \textbf{EvolCircle}  & \textbf{RPlate}  & \textbf{RMNIST} & \textbf{Portrait}  & \textbf{Cover Type} & \rebuttal{\textbf{FMoW}} & \textbf{Average}              \\
\midrule
GroupDRO                  & 75.5 $\pm$ 1.0   & 70.0 $\pm$ 4.9   & 76.5 $\pm$ 0.2   & 94.8 $\pm$ 0.1   & 66.4 $\pm$ 0.5   & \rebuttal{57.3 $\pm$ 0.1}   & \rebuttal{73.4}                                  \\
MLDG                      & 91.5 $\pm$ 2.0   & 66.9 $\pm$ 1.8   & 75.0 $\pm$ 0.3   & 66.2 $\pm$ 1.7   & 68.4 $\pm$ 0.7   & \rebuttal{43.8 $\pm$ 0.0}   & \rebuttal{68.6}                                  \\
MMD                       & 86.7 $\pm$ 5.7   & 59.9 $\pm$ 1.4   & 35.4 $\pm$ 0.0   & 95.4 $\pm$ 0.1   & 69.8 $\pm$ 0.4   & \rebuttal{60.0 $\pm$ 0.0}   & \rebuttal{67.8}                                  \\
SagNet                    & 78.7 $\pm$ 3.2   & 63.8 $\pm$ 2.9   & 79.4 $\pm$ 0.1   & 95.3 $\pm$ 0.1   & 65.3 $\pm$ 2.2   & \rebuttal{56.2 $\pm$ 0.1}   & \rebuttal{73.1}                                  \\
VREx                      & 82.9 $\pm$ 6.6   & 61.1 $\pm$ 2.6   & 79.4 $\pm$ 0.1   & 94.3 $\pm$ 0.2   & 66.0 $\pm$ 0.9   & \rebuttal{61.2 $\pm$ 0.0}   & \rebuttal{73.3}                                  \\
SD                        & 81.7 $\pm$ 4.3   & 65.3 $\pm$ 1.4   & 78.8 $\pm$ 0.1   & 95.1 $\pm$ 0.2   & 69.1 $\pm$ 0.9   & \rebuttal{55.2 $\pm$ 0.0}   & \rebuttal{74.2}                                  \\
IRM                       & 86.2 $\pm$ 3.0   & 67.2 $\pm$ 2.1   & 47.5 $\pm$ 0.4   & 94.4 $\pm$ 0.3   & 66.0 $\pm$ 1.0   & \rebuttal{58.8 $\pm$ 0.0}   & \rebuttal{70.0}                                  \\
Mixup                     & 91.5 $\pm$ 2.6   & 66.8 $\pm$ 1.8   & 81.3 $\pm$ 0.2   & \textbf{96.4 $\pm$ 0.2}   & 69.7 $\pm$ 0.6   & \rebuttal{59.5 $\pm$ 0.0}   & \rebuttal{77.5}                                  \\
CORAL                     & 86.8 $\pm$ 5.1   & 61.9 $\pm$ 1.4   & 78.4 $\pm$ 0.1   & 95.1 $\pm$ 0.1   & 68.1 $\pm$ 1.3   & \rebuttal{56.1 $\pm$ 0.0}   & \rebuttal{74.4}                                  \\
MTL                       & 77.7 $\pm$ 2.4   & 66.0 $\pm$ 1.2   & 77.2 $\pm$ 0.0   & 95.4 $\pm$ 0.1   & 69.2 $\pm$ 0.9   & \rebuttal{51.7 $\pm$ 0.0}   & \rebuttal{72.9}                                  \\
RSC                       & 91.5 $\pm$ 2.1   & 67.9 $\pm$ 4.2   & 74.7 $\pm$ 0.1   & 95.5 $\pm$ 0.1   & 69.4 $\pm$ 0.3   & \rebuttal{55.7 $\pm$ 0.1}   & \rebuttal{75.8}                                  \\
\rebuttal{DIRL} & \rebuttal{53.3 $\pm$ 0.2}   & \rebuttal{56.3 $\pm$ 0.4}   & \rebuttal{76.3 $\pm$ 0.3}   & \rebuttal{93.2 $\pm$ 0.2}   & \rebuttal{61.2 $\pm$ 0.3} & \rebuttal{43.4 $\pm$ 0.3}  & \rebuttal{64.0}\\
\rebuttal{Prototypical}& \rebuttal{93.6 $\pm$ 0.5}   & \rebuttal{66.3 $\pm$ 0.4}   & \rebuttal{85.2 $\pm$ 0.4}   & \rebuttal{96.2 $\pm$ 0.3}   & \rebuttal{66.5 $\pm$ 0.4}  & \rebuttal{53.3 $\pm$ 0.2} & \rebuttal{76.9}\\
ERM & 72.7 $\pm$ 1.1   & 63.9 $\pm$ 0.9   & 79.4 $\pm$ 0.0   & 95.8 $\pm$ 0.1   & 71.8 $\pm$ 0.2   & \rebuttal{54.6 $\pm$ 0.1}  & \rebuttal{74.7}                                  \\
DPNets (Ours)& \textbf{94.2 $\pm$ 0.9}   & \textbf{95.0 $\pm$ 0.5}   & \textbf{87.5 $\pm$ 0.1}   & \textbf{96.4 $\pm$ 0.0}   & \textbf{72.5 $\pm$ 1.0} & \rebuttal{\textbf{66.8 $\pm$ 0.1}}  & \rebuttal{\textbf{85.4}}\\
\bottomrule
\end{tabular}}
\end{center}
\end{table}

\subsection{Results and Analysis}

\textbf{Overall Evaluation} The performances of our proposed method and baselines are reported in Table \ref{tab:all_res}. It can be observed that DPNets consistently outperforms other baselines over all the data sets, and achieves $89.1\%$ on average which is significantly higher the other algorithms ($\approx 8\% - 20\%$). The results indicate that existing DG methods cannot deal with domain shifts well while DPNets can properly capture the evolving patterns in the environments. It is also worth noting that directly employing Prototypical Network on our problem setting does not receive good result (81.5\%), which further illustrates the effectiveness of our architectural design.


To further investigate the learning behaviors in evolving environments, we study the synthetic data sets, where the evolving pattern can be manually controlled. Here, we studied two typical evolving scenarios $P(X)$ and $P(Y|X)$, corresponding to the EvolCircle and RPlate data sets respectively. 

\textbf{Evolving $P(X)$ (EvolCircle).} 
The decision boundaries on the unseen target domain $\mathcal{D}_{30}$ learned by ERM and DPNets are shown in Fig.~\ref{fig:evolcircle}(c) and Fig.~\ref{fig:evolcircle}(d) respectively. 
We can observe that DPNets fits the ground truth significantly better than that of ERM. This indicates that our approach can capture the evolving pattern of $P(X)$ according to source domains and then learn a better classifier for the target domain. 
Furthermore, it can also be observed that the decision boundary learned by ERM achieves better performance on the observed source domains. This is because it focuses on improving generalization ability on all source domains, which leads the poor performance on the outer target domain $\mathcal D_{30}$. On the contrary,  DPNets can ``foresee" the prototypes for the target domain, which guarantees a good generalization performance even though it may not perform well on tge source domains. 


\textbf{Evolving $P(Y|X)$ (RPlate).} 
By visualizing the data sets, we can observe that the predicted boundary of DPNets better approximates the ground-truth, compared with the result of ERM. This indicates that our approach can also capture the $P(Y|X)$ evolving pattern. Existing DG methods perform poorly on this data set because the ground truth labeling function varies. Under the evolving labeling functions, even the same instance can have different labels in different domains. Thus, there does not exist a single model that can perform well across  all the domains. For this situation, learning a model specifically for one domain instead of all domains can be a possible solution. DPNets can capture the evolving pattern and produce a  model specifically for the target domain.

\begin{wrapfigure}[16]{r}{0.4\textwidth}
  \begin{center}
    \includegraphics[width=0.32\textwidth]{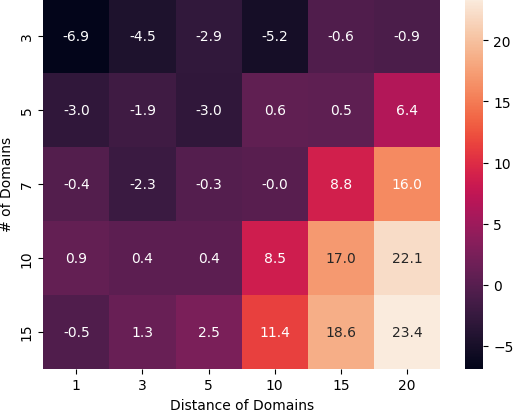}
  \end{center}
  \caption{\rebuttal{Performance of RMNIST w.r.t. different domain numbers and distances.}}
  \label{3d}
\end{wrapfigure}

\textbf{When to apply DPNets?} Existing DG methods assume that the distances among observed and unseen domains does not very large. However, the dissimilarity between domains is a crucial factor which can fundamentally influence the possibility and performance of generalization. To investigate the impact of variances of the environment, we create a series of variations on the raw RMNST data by jointly varying the number of domains (Table \ref{tab:domain_num_RMNIST_table}) and the degree interval (Table \ref{tab:domain_dis_RMNIST_table}) between two consecutive domains. The performance improvement of DPNets over the baseline ERM ($\text{Acc}_{\text{DPNets}} - \text{Acc}_{\text{ERM}}$) is shown in Fig. \ref{3d}. Experimental results indicate that DPNets performs better when the number of domains and the distance between domains increase 
On one hand, greater number of domains and larger distance between them lead to more significant difference across domains. 
This makes traditional DG methods harder to train one model from all domains, but oppositely more domains benefit our DPNets to learn the evolving pattern to achieve better performance.  On the other hand, we observed that the DPNets significantly outperforms other baselines when the number of domains and the distance between domains increase. 

\begin{table}[tbp]
\caption{Comparison of accuracy (\%) of different methods on RMNIST data set with different number of domains.}
\label{tab:domain_num_RMNIST_table}
\begin{center}
\adjustbox{max width=\textwidth}{%
\begin{tabular}{lccccccccc}
\toprule
\multicolumn{1}{c}{\# Domains} & 3                       & 5                       & 7                       & 9                       & 11                      & 13                      & 15                      & 17 &19                      \\
\midrule
Mixup                            & 82.3 $\pm$ 0.3          & \textbf{83.3 $\pm$ 0.3} & \textbf{83.8 $\pm$ 0.6}          & \textbf{83.3 $\pm$ 0.3}          & 80.0 $\pm$ 0.3          & 78.3 $\pm$ 0.3          & 80.0 $\pm$ 0.3          & 77.3 $\pm$ 0.3    & 72.5 $\pm$ 0.3        \\
IRM                              & 46.0 $\pm$ 0.2          & 44.0 $\pm$ 0.0          & 35.6 $\pm$ 0.3          & 46.5 $\pm$ 0.3          & 40.4 $\pm$ 0.4          & 49.6 $\pm$ 0.3          & 46.0 $\pm$ 0.1          & 46.9 $\pm$ 0.3   & 41.3 $\pm$ 0.1         \\
MLDG                             & \textbf{85.0 $\pm$ 0.2 }         & 81.9 $\pm$ 0.4          &82.7 $\pm$ 0.3          & 80.0 $\pm$ 0.6          & 79.0 $\pm$ 0.1          & 74.2 $\pm$ 0.1          & 77.9 $\pm$ 0.3          & 71.7 $\pm$ 0.1  & 68.8 $\pm$ 0.3          \\
ERM                             & 80.0 $\pm$ 0.3          & 81.6 $\pm$ 0.3          & 81.3 $\pm$ 0.1          & 79.7 $\pm$ 0.2          & 79.7 $\pm$ 0.3          & 75.6 $\pm$ 0.3          & 77.8 $\pm$ 0.3          & 69.1 $\pm$ 0.3   & 74.4 $\pm$ 0.1         \\
DPNets (Ours)                              & 83.4 $\pm$ 0.1 & 83.1 $\pm$ 0.3 & 81.1 $\pm$ 0.1 & 82.8 $\pm$ 0.3 & \textbf{88.1 $\pm$ 0.3} & \textbf{\rebuttal{87.3} $\pm$ 0.5} & \textbf{86.6 $\pm$ 0.1} & \textbf{85.6 $\pm$ 0.3} & \textbf{86.3 $\pm$ 0.3}\\
\bottomrule
\end{tabular}
}
\end{center}
\end{table}

\begin{table}[tbp]
\caption{Comparison of accuracy (\%) of different methods on RMNIST data set with different distance between domains.  }
\label{tab:domain_dis_RMNIST_table}
\begin{center}
\adjustbox{max width=\textwidth}{%
\begin{tabular}{lcccccc}
\toprule
\multicolumn{1}{c}{Domain Distance}& $3^\circ$                       & $5^\circ$                       & $7^\circ$                       & $10^\circ$                      & $15^\circ$                      & $20^\circ$                      \\
\midrule
Mixup                               & 92.5 $\pm$ 0.1          & \textbf{91.9 $\pm$ 0.3}          & \textbf{88.4 $\pm$ 0.3}          & 81.3 $\pm$ 0.2          & 73.1 $\pm$ 0.1          & 59.4 $\pm$ 0.3          \\
IRM                                     & 69.4 $\pm$ 0.2          & 63.4 $\pm$ 0.1          & 49.7 $\pm$ 0.0          & 47.5 $\pm$ 0.4          & 35.6 $\pm$ 0.3          & 31.3 $\pm$ 0.3          \\
MLDG                                & 90.9 $\pm$ 0.1          & 87.5 $\pm$ 0.2          & 85.0 $\pm$ 0.2          & 75.0 $\pm$ 0.3          & 71.9 $\pm$ 0.1          & 56.9 $\pm$ 0.3          \\
ERM                                    & 92.2 $\pm$ 0.1          & 88.8 $\pm$ 0.0          & 82.8 $\pm$ 0.1          & 79.4 $\pm$ 0.0          & 66.3 $\pm$ 0.1          & 53.1 $\pm$ 0.3          \\
DPNets (Ours)                             & 91.9 $\pm$ 0.3 & 91.3 $\pm$ 0.3 & \textbf{88.4 $\pm$ 0.3} & \textbf{87.5 $\pm$ 0.1} & \textbf{85.6 $\pm$ 0.3} & \textbf{83.8 $\pm$ 0.3}\\
\bottomrule
\end{tabular}
}
\end{center}
\end{table}

In Table \ref{tab:domain_num_RMNIST_table} and Table \ref{tab:domain_dis_RMNIST_table}, we respectively analyzed the affect of \emph{the domain distance} and \emph{the number of domains} on the generalization ability of different models. We can observe that, with the increasing complexity of source domains, the DPNets benefits a lot from the evolving information and the performance gap between our method and baselines increase. 
In addition, from Table  \ref{tab:domain_num_RMNIST_table} we can find that the performance of traditional DG methods fluctuates when the number of domains increases. As for the DPNets , its performance continuously improves when the domain number increases, since it easily learns the evolving manners from more domains. Please refer Section more discussion about this.
From Table \ref{tab:domain_dis_RMNIST_table}, we can see that when the domain distance increases, the performance of DG methods decreases severely while the performance of DPNets drops slightly. 

In conclusion, the experimental results imply that it is hard for traditional DG methods to solve the EDG problem when the number of domains and distance between domains increase, while our DPNets can still perform well in such a scenario. 

\begin{figure*}[htbp]
		\centering 
		\subfloat[ERM]{\includegraphics[width=0.3\textwidth]{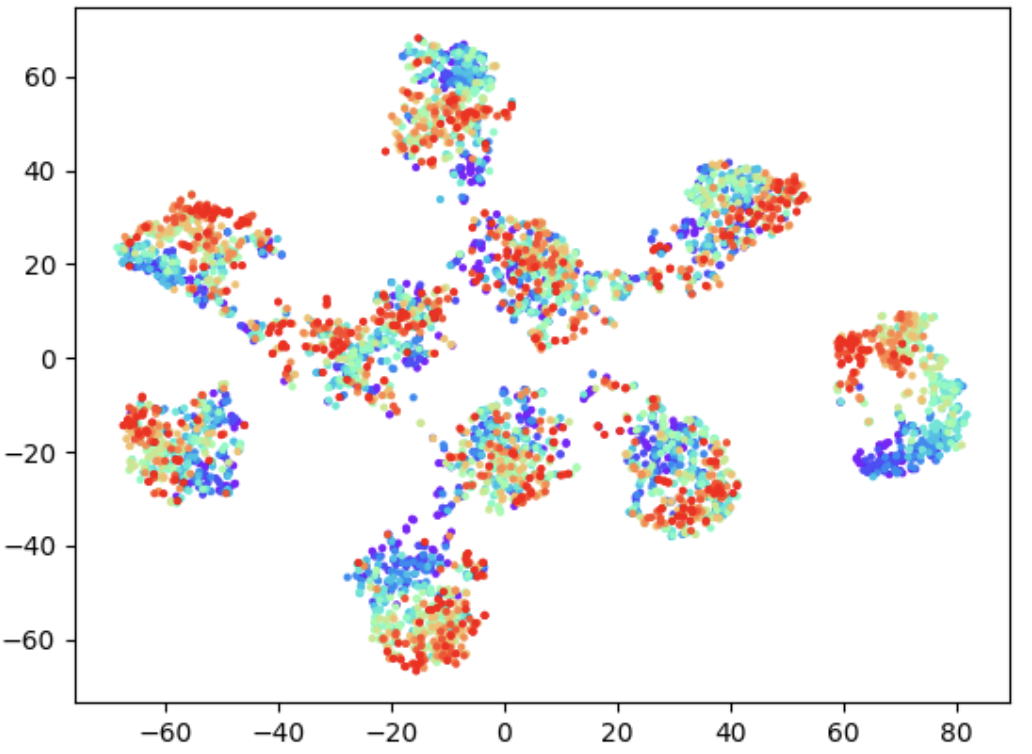}}
		\subfloat[Prototypical Network]{\includegraphics[width=0.3\textwidth]{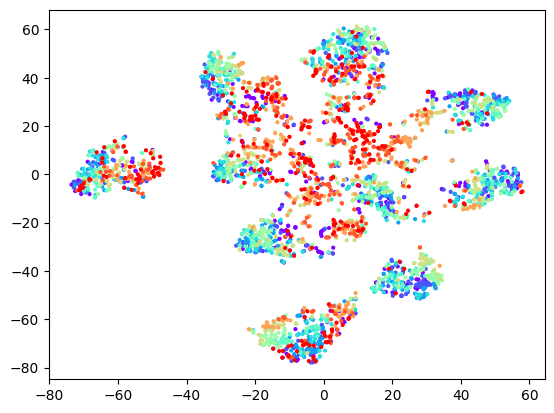}}
		\subfloat[DPNets]{\includegraphics[width=0.3\textwidth]{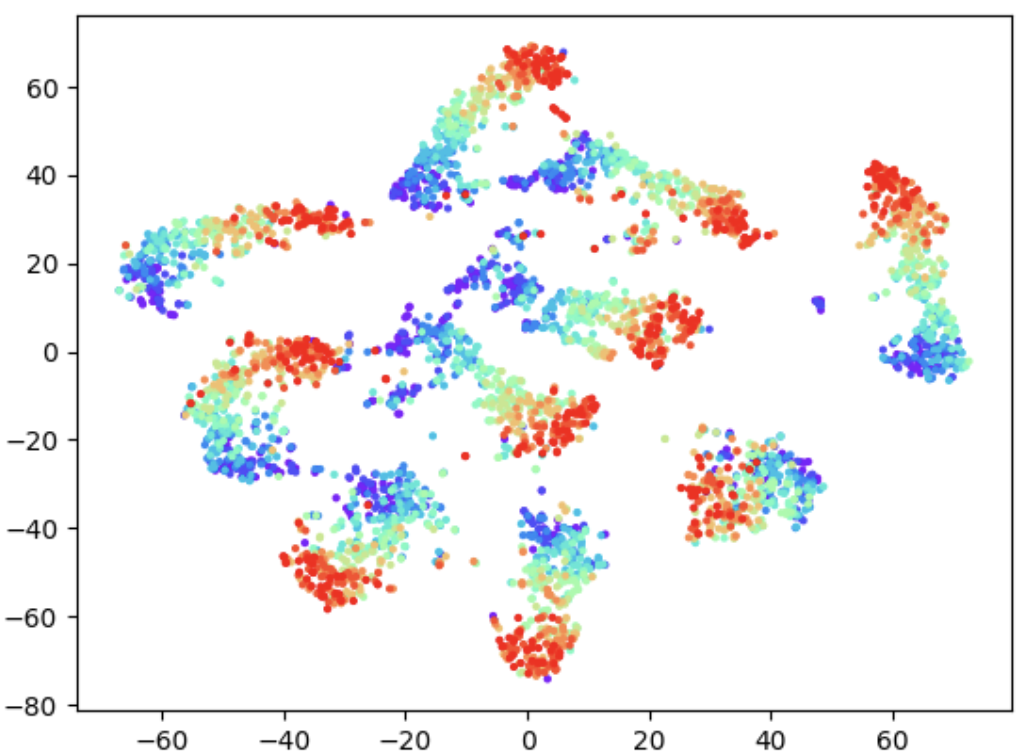}}
\caption{T-SNE Visualization of embedded RMNIST data learned by ERM and DPNets.}
\label{tsne}
\end{figure*}

\textbf{T-SNE visualization over evolving domains.} In this part, we investigate the ability of DPNets and DG methods in distilling domain evolving information. Most domain generalization approaches aim to learn an invariant representation across all domains. While in the EDG scenario, we need to leverage the evolving pattern to improve the generalization process. Here, we use t-SNE to visualize, respectively, the representations learned from the second-to-last layer by ERM, DPNets\rebuttal{, and original Prototypical Network} in Fig. \ref{tsne}. The colors from red to blue correspond to the domains index from 1 to 30. The feature visualizations demonstrate that DPNets can keep the domain evolving even in the last layer of the network, which makes it possible to leverage that knowledge. On the contrary, the evolving pattern learned by \rebuttal{Prototypical Network and ERM is less obvious. The results further prove that ERM and original Prototypical Network can not leverage evolving information well.}

\section{Conclusions}
In this paper, we study the problem of domain generalization in an evolving environment, and propose evolving domain generalization (EDG)  as a general framework to address it. Our theoretical analysis highlights the role of learning a mapping function to capture the evolving pattern over domains. Based on our analysis, we propose directional prototypical networks (DPNets), a simple and efficient algorithm for EDG. Experiments on both synthetic and real-world data sets validate the effectiveness of our method. 



\clearpage

\newpage
\rebuttal{
\section*{Ethics Statement}
This paper presents an algorithm that can exploit evolving information in a continuously changing environment to improve the performance of the model in target domains where data are not available. the proposed approach may also introduce the potential negative impact: the Portrait dataset we use is only intended to demonstrate algorithm's superior performance on classification tasks.
}
\bibliography{example_paper.bib}
\bibliographystyle{ICLR/iclr2021_conference}

\newpage
\appendix

\section{Proofs}

\subsection{Lemma~\ref{lemmastart}}

We first prove an intermediate lemma:

\begin{lemma}\label{lemma1}
Let $z\in\mathcal{Z}=\mathcal{X}\times\mathcal{Y}$ be the real valued integrable random variable, let $P$ and $Q$ are two distributions on a common space $\mathcal{Z}$ such that $Q$ is absolutely continuous w.r.t. $P$. If for any function $f$ and $\lambda\in\R$ such that $\E_P[e^{\lambda(f(z)-\E_P(f(z))}]<\infty$, then we have:
\begin{equation*}
   \lambda (\E_{Q} f(z) - \E_{P}f(z)) \leq D_{\text{KL}}(Q\|P) + \log \E_P[e^{\lambda(f(z)-\E_P(f(z))}],
\end{equation*}
where $D_{\text{KL}}(Q\|P)$ is the Kullback–Leibler divergence between distribution $Q$ and $P$, and the equality arrives when $f(z)= \E_{P} f(z) + \frac{1}{\lambda}\log(\frac{d Q}{d P})$.
\end{lemma}

\begin{proof}
We let $g$ be {any} function such that $\E_P[e^{g(z)}]<\infty$, then we define a random variable $Z_g(z) = \frac{e^{g(z)}}{\E_P[e^{g(z)}]}$, then we can verify that $\E_{P}(Z_g) =1$. We assume another distribution $Q$ such that $Q$ (with distribution density $q(z)$) is absolutely continuous w.r.t. $P$ (with distribution density $p(z)$), then we have:
\begin{equation*}
\begin{split}
     \E_{Q}[\log Z_g] & = \E_{Q}[\log\frac{q(z)}{p(z)} + \log(Z_g\frac{p(z)}{q(z)})]  = D_{\text{KL}}(Q\|P) + \E_{Q}[\log(Z_g\frac{p(z)}{q(z)})]\\
     & \leq D_{\text{KL}}(Q\|P) + \log\E_{Q}[\frac{p(z)}{q(z)}Z_g]= D_{\text{KL}}(Q\|P) + \log \E_{P}[Z_g]
\end{split}
\end{equation*}
Since $\E_{P}[Z_g] = 1$ and according to the definition we have $\E_{Q}[\log Z_g] = \E_{Q}[g(z)] - \E_{Q}\log\E_{P}[e^{g(z)}] = \E_{Q}[g(z)] - \log\E_{P}[e^{g(z)}]$ (since $\E_{P}[e^{g(z)}]$ is a constant w.r.t. $Q$) and we therefore have:
\begin{equation}
    \E_{Q}[g(z)] \leq  \log\E_{P}[e^{g(z)}] + D_{\text{KL}}(Q\|P)
    \label{change_of_measure}
\end{equation}
Since this inequality holds for any function $g$ with finite moment generation function, then we let $g(z) = \lambda( f(z)-\E_P f(z))$ such that $\E_P[e^{f(z)-\E_P f(z)}]<\infty$. Therefore we have $\forall \lambda$ and $f$ we have:
\begin{equation*}
    \E_{Q}\lambda(f(z)-\E_P f(z)) \leq D_{\text{KL}}(Q\|P) + \log \E_P[e^{\lambda(f(z)-\E_P f(z)}]
\end{equation*}
Since we have $\E_{Q}\lambda(f(z)-\E_P  f(z)) = \lambda \E_{Q} (f(z)-\E_P f(z))) = \lambda (\E_{Q} f(z) - \E_{P} f(z))$, therefore we have:
\begin{equation*}
   \lambda (\E_{Q} f(z) - \E_{P} f(z)) \leq D_{\text{KL}}(Q\|P) + \log \E_P[e^{\lambda (\E_{Q} f(z) - \E_{P} f(z))}]
\end{equation*}
As for the attainment in the equality of Eq.(\ref{change_of_measure}), we can simply set  $g(z) = \log(\frac{q(z)}{p(z)})$, then we can compute $\E_{P}[e^{g(z)}]=1$ and the equality arrives. Therefore in Lemma 1, the equality reaches when $\lambda(f(z)- \E_{P} f(z)) = \log(\frac{d Q}{d P})$.
\end{proof}
In the classification problem, we define the observation pair $z=(x,y)$. We also define the loss function $\ell(z)=L\circ h(z)$ with deterministic hypothesis $h$ and prediction loss function $L$. Then for abuse of notation, we simply denote the loss function $\ell(z)$ in this part.

Given Lemma~\ref{lemma1}, we are ready to prove Lemma~\ref{lemmastart}.
\begin{proof}
According to Lemma~\ref{lemma1}, $\forall \lambda>0$ we have: 
\begin{equation}
    \E_Q f(z) - \E_P f(z) \leq \frac{1}{\lambda} (\log\E_{P}~e^{[\lambda(f(z)-\E_{P}f(z))]} + D_{\text{KL}}(Q\|P))
    \label{sub_1}
\end{equation}

\noindent And $\forall \lambda<0$ we have:
\begin{equation}
    \E_Q f(z) - \E_P f(z) \geq \frac{1}{\lambda} (\log\E_{P}~e^{[\lambda(f(z)-\E_{P}f(z))]} + D_{\text{KL}}(Q\|P))
    \label{sub_2}
\end{equation}

Then we introduce an intermediate distribution $\mathcal{M}(z) = \frac{1}{2}(\mathcal{D}(z) + \mathcal{D}'(z))$, then $\text{supp}(\mathcal{D})\subseteq\text{supp}(\mathcal{M})$ and $\text{supp}(\mathcal{D}')\subseteq\text{supp}(\mathcal{M})$, and let $f=\ell$. Since the random variable $\ell$ is bounded through $G =  \max(\ell) -\min(\ell)$, then according to \cite{wainwright2019high} (Chapter 2.1.2), $\ell-\E_{P}\ell$ is sub-Gaussian with parameter at most $\sigma = \frac{G}{2}$, then we can apply Sub-Gaussian property to bound the $\log$ moment generation function:
\begin{equation*}
    \log\E_{P}~e^{[\lambda(\ell(z)-\E_{P}\ell(z))]} \leq \log e^{\frac{\lambda^2\sigma^2}{2}} \leq \frac{\lambda^2G^2}{8}.
\end{equation*}

\noindent In Eq.(\ref{sub_1}), we let $Q = \mathcal{D}'$ and $P=\mathcal{M}$, then $\forall \lambda>0$ we have:
\begin{equation}
    \E_{\mathcal{D}'}~\ell(z) - \E_{\mathcal{M}}~\ell(z) \leq \frac{G^2\lambda}{8} +  \frac{1}{\lambda}D_{\text{KL}}(\mathcal{D}'\|\mathcal{M})
    \label{sub_3}
\end{equation}

\noindent In Eq.(\ref{sub_2}), we let $Q = \mathcal{D}$ and $P=\mathcal{M}$, then $\forall \lambda<0$ we have:
\begin{equation}
    \E_{\mathcal{D}}~\ell(z) - \E_{\mathcal{M}}~\ell(z) \geq \frac{G^2\lambda}{8} +  \frac{1}{\lambda}D_{\text{KL}}(\mathcal{D}\|\mathcal{M})
    \label{sub_4}
\end{equation}

\noindent In Eq.(\ref{sub_3}), we denote $\lambda=\lambda_0>0$ and $\lambda=-\lambda_0<0$ in Eq.(\ref{sub_4}).
Then Eq.(\ref{sub_3}),  Eq.(\ref{sub_4}) can be reformulated as:
\begin{equation}
\begin{split}
    & \E_{\mathcal{D}'}~\ell(z) - \E_{\mathcal{M}}~\ell(z) \leq \frac{G^2\lambda_0}{8} +  \frac{1}{\lambda_0}D_{\text{KL}}(\mathcal{D}'\|\mathcal{M})\\
    & \E_{\mathcal{M}}~\ell(z) - \E_{\mathcal{D}}~\ell(z) \leq \frac{G^2\lambda_0}{8} +  \frac{1}{\lambda_0}D_{\text{KL}}(\mathcal{D}\|\mathcal{M})
\end{split}
    \label{sub_5}
\end{equation}
Adding the two inequalities in Eq.(\ref{sub_5}), we therefore have:
\begin{equation}
    \E_{\mathcal{D}'}~\ell(z)  \leq \E_{\mathcal{D}}~\ell(z) + \frac{1}{\lambda_0} \big(D_{\text{KL}}(\mathcal{D}\|\mathcal{M}) + D_{\text{KL}}(\mathcal{D}'\|\mathcal{M}) \big) + \frac{\lambda_0}{4}G^2 
\end{equation}
Since the inequality holds for $\forall \lambda_0$, then by taking $\lambda_0 = \frac{2}{G}\sqrt{D_{\text{KL}}(\mathcal{D}\|\mathcal{M}) + D_{\text{KL}}(\mathcal{D}'\|\mathcal{M})}$ we finally have:
\begin{equation}
      \E_{\mathcal{D}'}~\ell(z)  \leq \E_{\mathcal{D}}~\ell(z) + \frac{G}{\sqrt{2}}\sqrt{D_{\text{JS}}(\mathcal{D}'\|\mathcal{D})}
    \label{sub_6}
\end{equation}
Let $\mathcal{D}' = \mathcal{D}_t$ and $\mathcal{D} = \mathcal{D}_t^g$, we complete our proof.
\end{proof}

\subsection{Theorem~\ref{theoremds}}


\begin{proof}

According to Definition of $\lambda$-consistency, we have:

\begin{align*}
d_{JS}(\mathcal D^{g^*}_t||\mathcal D_t)\leq d_{JS}(\mathcal D^{g^*}_2||\mathcal D_2)+|d_{JS}(\mathcal D^{g^*}_t||\mathcal D_t)-d_{JS}(\mathcal D_2^{g^*}||\mathcal D_2)|\leq d_{JS}(\mathcal D^{g^*}_2||\mathcal D_2)+\lambda
\end{align*}

Similarly, we have the followings:
\begin{align*}
&d_{JS}(\mathcal D^{g^*}_t||\mathcal D_t)\leq d_{JS}(\mathcal D^{g^*}_i||\mathcal D_i)+|d_{JS}(\mathcal D^{g^*}_t||\mathcal D_t)-d_{JS}(\mathcal D_i^{g^*}||\mathcal D_i)|\leq d_{JS}(\mathcal D^{g^*}_i||\mathcal D_i)+\lambda \\
& \hspace{180pt} \cdots\\
&d_{JS}(\mathcal D^{g^*}_t||\mathcal D_t)\leq d_{JS}(\mathcal D^{g^*}_m||\mathcal D_m)+|d_{JS}(\mathcal D^{g^*}_t||\mathcal D_t)-d_{JS}(\mathcal D_m^{g^*}||\mathcal D_m)|\leq d_{JS}(\mathcal D^{g^*}_m||\mathcal D_m)+\lambda,
\end{align*}
which gives us
\begin{align*}
d_{JS}(\mathcal D_t^{g^*}||\mathcal D_t)\leq \frac{1}{m-1}\sum _{i=2}^{m}d_{JS}(\mathcal D_i^{g^*}||\mathcal D_{i})+\lambda
\end{align*}
Then, according to Lemma~\ref{lemmastart}, we have
\begin{align*}
R_{\mathcal D_{t}}(h)&\leq R_{\mathcal D_t ^{g^*}}(h)+\frac{G}{\sqrt{2}}\sqrt{d_{JS}(\mathcal D_t^{g^*}||\mathcal D_t)}\\
&\leq R_{\mathcal D_t ^{g^*}}(h)+\frac{G}{\sqrt{2}}\sqrt{\frac{1}{m-1}\sum _{i=2}^{m}d_{JS}(\mathcal D_i^{g^*}||\mathcal D_{i})+\lambda}\\
& \le R_{\mathcal{D}_t^{g^*}}(h)+\frac{G}{\sqrt{2(m-1)}}\Bigg(\sqrt{\sum _{i=2}^{m}d_{JS}(\mathcal D_{i}^{g^*}||\mathcal D_{i})}+\sqrt{(m-1)\lambda}\Bigg)
\end{align*}
\end{proof}






\subsection{Corollary~\ref{corollay1}}
We first introduce the upper bound for Jensen Shannon (JS) Divergence decomposition:
\begin{lemma}
\label{jsd_decompose}
Let $\mathcal D(x,y)$ and $\mathcal D'(x,y)$ be two distributions over $\mathcal X \times \mathcal Y$, $\mathcal D(y)$ and $\mathcal D'(y)$ be the corresponding marginal distribution of $y$, $\mathcal D(x|y)$ and $\mathcal D'(x|y)$ be the corresponding conditional distribution given $y$, then we can get the following bound,
\begin{align*}
d_{JS}(\mathcal D(x,y)||\mathcal D'(x,y))& \leq  \\
& \hspace{-54pt} d_{JS}(\mathcal D(y)||\mathcal D'(y))+\mathbb E_{y\sim \mathcal D(y)}d_{JS}(\mathcal D(x|y)||\mathcal D'(x|y))+\mathbb E_{y\sim \mathcal D'(y)}d_{JS}(\mathcal D(x|y)||\mathcal D'(x|y))
\end{align*}
\end{lemma}

\begin{proof}

Let $\mathcal M(x,y)=\frac{1}{2}(\mathcal D(x,y)+\mathcal D'(x,y))$, then we have
\begin{align*}  
2\cdot d_{JS}(\mathcal D(x,y)||\mathcal D'(x,y))&=d_{KL}(\mathcal D(x,y)||\mathcal M(x,y))+d_{KL}(\mathcal D'(x,y)||\mathcal M(x,y))\\&=d_{KL}(\mathcal D(y)||\mathcal M (y))+\mathbb E_{y\sim \mathcal D(y)} d_{KL}(\mathcal D(x|y)||\mathcal M(x|y))\\&+d_{KL}(\mathcal D'(y)||\mathcal M(y))+\mathbb E_{y\sim \mathcal D'(y)} d_{KL}(\mathcal D(x|y)||\mathcal M(x|y))\\&=2\cdot d_{JS}(\mathcal D(y)||\mathcal D'(y))+\mathbb E_{y\sim \mathcal D(y)} d_{KL}(\mathcal D(x|y)||\mathcal M(x|y))\\&+\mathbb E_{y\sim \mathcal D'(y)} d_{KL}(\mathcal D(x|y)||\mathcal M(x|y))
\end{align*}

To bound the last two terms with JS divergence, we have:

\begin{align}\label{eqq1}
d_{KL}(\mathcal D(x|y)||\mathcal M(x|y))&\leq d_{KL}(\mathcal D(x|y)||\mathcal M(x|y))+d_{KL}(\mathcal D'(x|y)||\mathcal M(x|y))\\&=2\cdot d_{JS}(\mathcal D(x|y)||\mathcal D'(x|y)). \nonumber
\end{align} 
Also,
\begin{align}\label{eqq2}
d_{KL}(\mathcal D'(x|y)||\mathcal M(x|y))&\leq d_{KL}(\mathcal D(x|y)||\mathcal M(x|y))+d_{KL}(\mathcal D'(x|y)||\mathcal M(x|y))\\&=2\cdot d_{JS}(\mathcal D(x|y)||\mathcal D'(x|y)). \nonumber
\end{align}
Combining (\ref{eqq1}) and (\ref{eqq2}) gives us
\begin{align*}
d_{JS}(\mathcal D(x,y)||\mathcal D'(x,y))& \leq  \\
& \hspace{-54pt} d_{JS}(\mathcal D(y)||\mathcal D'(y))+\mathbb E_{y\sim \mathcal D(y)}d_{JS}(\mathcal D(x|y)||\mathcal D'(x|y))+\mathbb E_{y\sim \mathcal D'(y)}d_{JS}(\mathcal D(x|y)||\mathcal D'(x|y)),
\end{align*}
which concludes the proof.
\end{proof}
Given Lemma~\ref{jsd_decompose}, we are ready to prove Corollary~\ref{corollay1}.

\begin{proof}

\begin{align*}   R_{\mathcal{D}_t}(h) &\le R_{\mathcal D_t ^{g^*}}(h)+\frac{G}{\sqrt{2(m-1)}}\sqrt{\sum_{i=2}^{m}{d_{JS}(\mathcal D_i^{g^*}||\mathcal D_{i})}}+G\sqrt{\frac{\lambda}{2}}\\&\le R_{\mathcal D_t ^{g^*}}(h)+G\sqrt{\frac{\lambda}{2}}+\frac{G}{\sqrt{2(m-1)}}\cdot \\&\sqrt{\sum_{i=2}^{m}{d_{JS}(\mathcal D_i^{g^*}(y)||\mathcal D_i(y))+\mathbb E_{y\sim \mathcal D_i(y)}d_{JS}(\mathcal D_i^{g^*}(x|y)||\mathcal D_i(x|y))+\mathbb E_{y\sim \mathcal D_i^{g^*}(y)}d_{JS}(\mathcal D_i^{g^*}(x|y)||\mathcal D_i(x|y))}}\\
&\le R_{\mathcal D_t^{g^*}}(h)+\frac{G}{\sqrt{2(m-1)}}\Bigg({\sqrt{\sum _{i=2}^{m}d_{JS}(\mathcal D^{g^*}_{i}(y)||\mathcal D_{i}(y))}} + \sqrt{(m-1)\lambda}\\
& \hspace{24pt}+{\sqrt{\sum _{i=2}^{m}\mathbb E_{y\sim \mathcal D_{i}^{g^*}(y)}d_{JS}(\mathcal D^{g^*}_{i}(x|y)||\mathcal D_{i}(x|y))}}+{\sqrt{\sum _{i=2}^{m}\mathbb E_{y\sim \mathcal D_{i}(y)}d_{JS}(\mathcal D^{g^*}_{i}(x|y)||\mathcal D_{i}(x|y))}}\Bigg).
\end{align*}

\end{proof}]

\rebuttal{

\rebuttal{
\subsection{Comparison to the Assumptions of Existing Studies}

Learning in a non-stationary environment is impossible if no assumption is imposed on the environment. Existing theoretical studies of evolving domain adaptation have made various assumptions on the evolving pattern of the environment to obtain meaningful results. Specifically, \cite{usgda} assumes that $\rho(\mathcal D_t,\mathcal D_{t+1})<\epsilon$, where $\rho(\cdot,\cdot)$ is some distance measurement of distribution, and the assumption in \cite{laed} is $d_{\mathcal H\Delta\mathcal H}(\mathcal D_{t_1},\mathcal D_{t_2})\leq \alpha|t_1-t_2|$. In other words, they both assume that the distance between two consecutive domains is small. Although such an assumption seems intuitive and reasonable, there are two fundamental issues:

\begin{enumerate}
  \item {\bf Too restrictive for real-world scenarios.}  In many problems, the distance between two domains can be much larger than the difference of domain indices. For example, a small angular rotation may result in a large difference of the pixel-level data distribution. Existing assumptions will fail to characterize such a scenario since both $\rho$ and $d_{\mathcal H\Delta\mathcal H}$ can be quite large, but this problem is still learnable in practice.
  \item {\bf Not taking the algorithm into account.}  Both $\rho$ and $d_{\mathcal H\Delta\mathcal H}$ are algorithm-independent in the sense that they only characterize the nature of an environment itself but does not involve any specific learning algorithm. Consequently, these assumptions cannot directly motivate any concrete strategies for learning the evolving pattern. 
\end{enumerate}

In contrast, our notion of $\lambda$-consistency: $|d_{JS}(\mathcal D_i^{g^*}||\mathcal D_i)-d_{JS}(\mathcal D_j^{g^*}||\mathcal D_j)| \le \lambda$ offers natural solutions to these issues: 
\begin{enumerate}
    \item It reveals that what really matters is not the distance between two consecutive domains but the \emph{stability} (predictability) of the evolving pattern of an environment. Specifically, if the evolving pattern of an environment is stable (not necessarily slow), there will exist a mapping function such that $\lambda$ is small. In other words, our notion indicates that generalization performance can still be guaranteed even though the distance between two consecutive domains is large, as long as $\lambda$ is small.
    \item It also highlights the role of the mapping function $g$. Since $g^*$ is unknown in practice, one primary objective of a EDG algorithm is essentially to minimize the distance between the real and mapped domains. In our implementation (i.e., DPNets), this objective is realized by estimating the prototypes of the next domain by leveraging the instances from the previous domain. Other realizations are also possible, which opens up avenues for future work.
\end{enumerate}

}

\section{Additional Experiments}
\subsection{Further Investigation of Interpolation and Extrapolation}

In Table \ref{tab:domain_num_RMNIST_table}, we can observe that the performances of ERM are not improved as the number of domains increases, which is counter-intuitive. We speculate that this is due to the ``extrapolation" nature of EDG.  To further investigate its impact on the generalization performance on the target domain, we compare the following three settings on the RMNIST dataset: 

\begin{enumerate}
  \item \textbf{DPNets-Evolving (Extrapolation).} Same as DPNets in Section~\ref{sec:exp}, where the target domain $\mathcal{D}_t = \mathcal{D}_{i+1}$.
  \item \textbf{ERM-Evolving (Extrapolation).} Same as ERM in Section~\ref{sec:exp}, where the target domain $\mathcal{D}_t = \mathcal{D}_{i+1}$.
  \item \textbf{ERM-Interpolation.} The ERM approach using the domain in the middle as the target domain, and the rest domains as the source domains. 
\end{enumerate}

Note that (1) and (2) are different algorithms with the same problem setup, and (2) and (3) use the same algorithm but with different problem setups.

We vary the the numbers of domain numbers and domain distances, and the results are reported in Fig.~\ref{erm_distance}, from which we have the following observations:

\begin{enumerate}
    \item The overall trend of the DPNets  is  going up as the number of domains increases. 
    \item The performances of ERM-Evolving do not increase as a function of the number of domain distance, which is consistent with the results in Table~\ref{tab:domain_num_RMNIST_table}.
    \item The performances of ERM-Interpolation increase as a function of the number of domain distance
    \item The improvements of  DPNets and ERM-Interpolation are not obvious once having sufficient amount of domains (e.g., \# of domains = 7 for ERM-Interpolation). We conjecture that it is because the evolving pattern of RMNIST is relatively simple. Thus, a small number of domains are sufficient to learn such a pattern, and increasing the number of domains may not necessarily improve the performances of DPNets and  ERM-Interpolation anymore. 
\end{enumerate}

The results indicate for extrapolation, having more domains does not necessarily help learn an invariant representation if we do not leverage the evolving pattern. Intuitively, as the target domain is on the ``edge" of the domains, having more domains also indicates that it is further away from the “center” of the source domains, which may even make the generalization even more challenging. On the other hand, if the target domain is ``among" the source domains (i.e., when we perform ``interpolation"), the source domains may act as augmented data which improve the generalization performance. In other words, if the more source domains will be beneficial for ``interpolation" but not necessarily for ``extrapolation" if the evolving pattern is not properly exploited.



\begin{figure*}[t]
		\centering 
		\subfloat[Distance = $3^\circ$]{\includegraphics[width=0.25\textwidth]{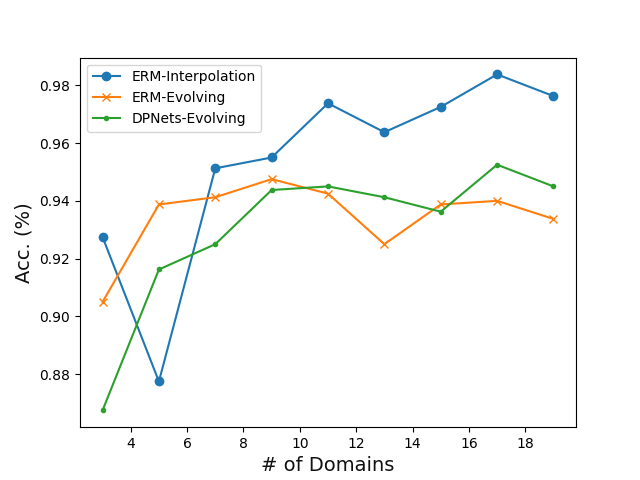}\label{erm_distance:3}}
		\subfloat[Distance = $7^\circ$]{\includegraphics[width=0.25\textwidth]{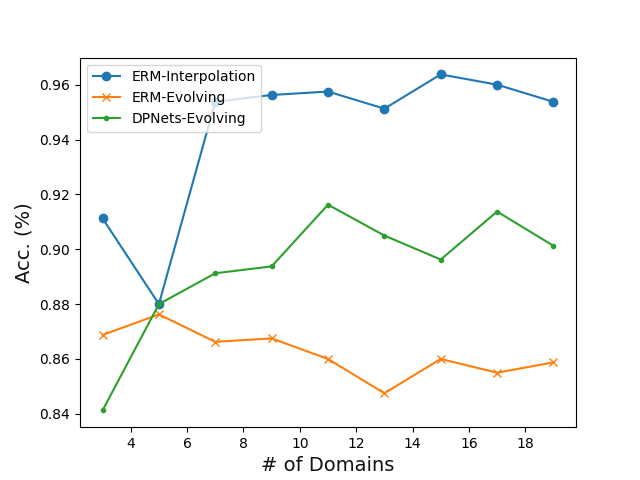}\label{erm_distance:7}}
		\subfloat[Distance = $11^\circ$]{\includegraphics[width=0.25\textwidth]{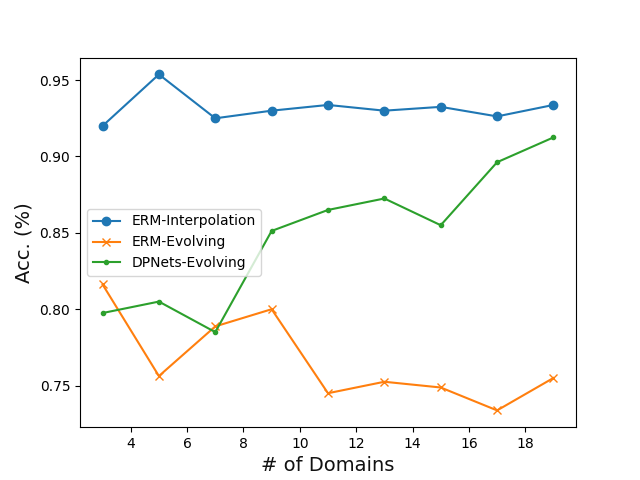}\label{erm_distance:11}}
		\subfloat[Average]{\includegraphics[width=0.25\textwidth]{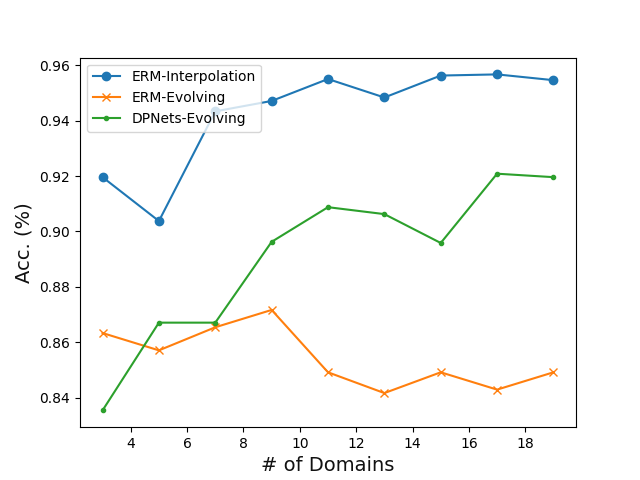}\label{erm_distance:ave}}
\caption{Performance of algorithms when numbers of domains changes.}
\label{erm_distance}
\end{figure*}

}

\rebuttal{

\subsection{Incorporating Domain Information into ERM.}

The ERM in Section~\ref{sec:exp} does not leverage the index information of the source domains. In order to make a more fair comparison, we incorporate the index information into ERM. Specifically, we investigate three strategies for incorporating the index information used in the literature: (1) Index Concatenation (Fig. \ref{index_info:one}), where the domain index is directly concatenated as a one-dimension feature \citep{li2021learning}; (2) One-hot Concatenation (Fig. \ref{index_info:two}), where the domain index is first one-hot encoded and then concatenated to the original features \citep{long2017conditional}; (3) Outer product (Fig. \ref{index_info:three}), where flattened the outer product of original features and the one-hot indexes is used as the final input \citep{shui2021benefits}.


\begin{figure*}[b]
		\centering 
		\subfloat[Index Concatenation ]{\includegraphics[width=0.33\textwidth]{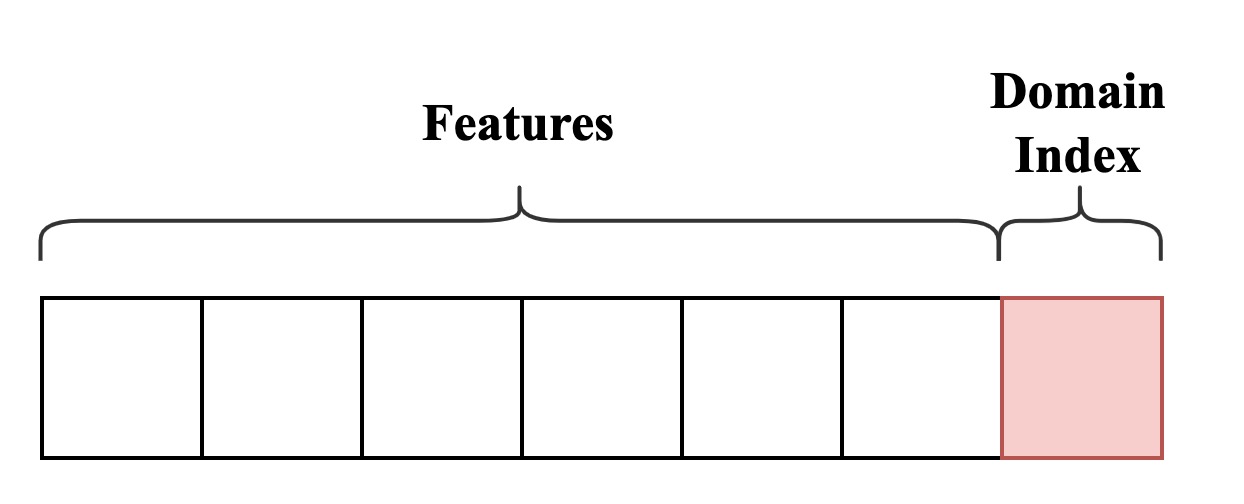}\label{index_info:one}}
		\subfloat[One-hot Concatenation]{\includegraphics[width=0.33\textwidth]{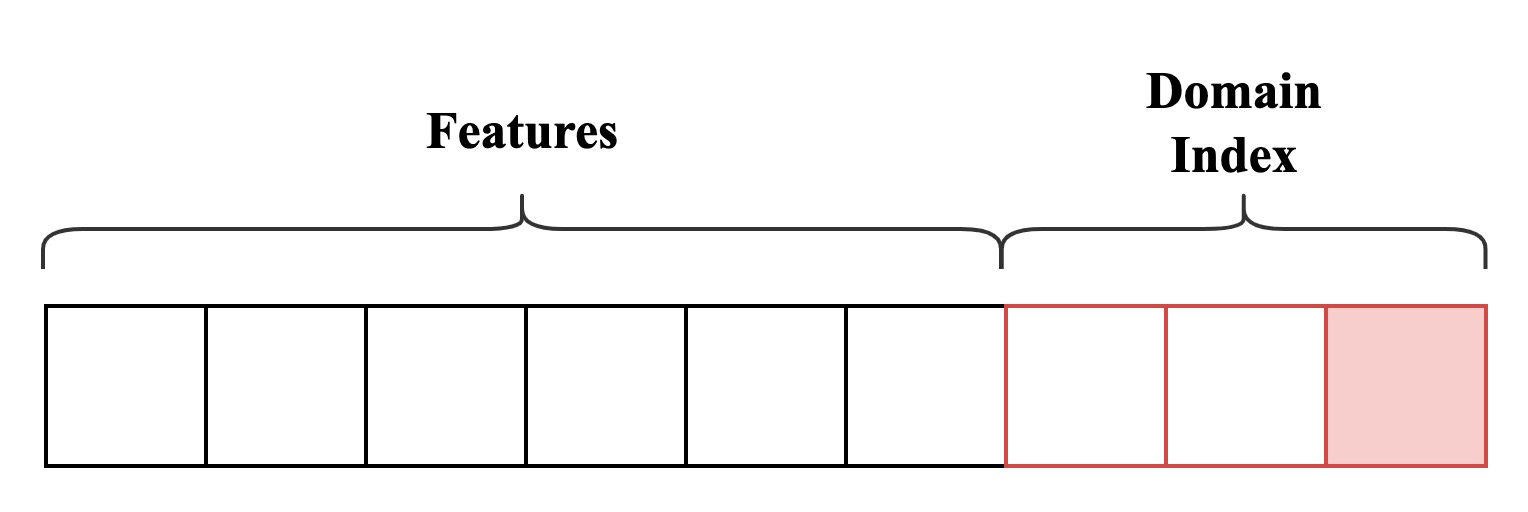}\label{index_info:two}}
		\subfloat[Outer Product]{\includegraphics[width=0.33\textwidth]{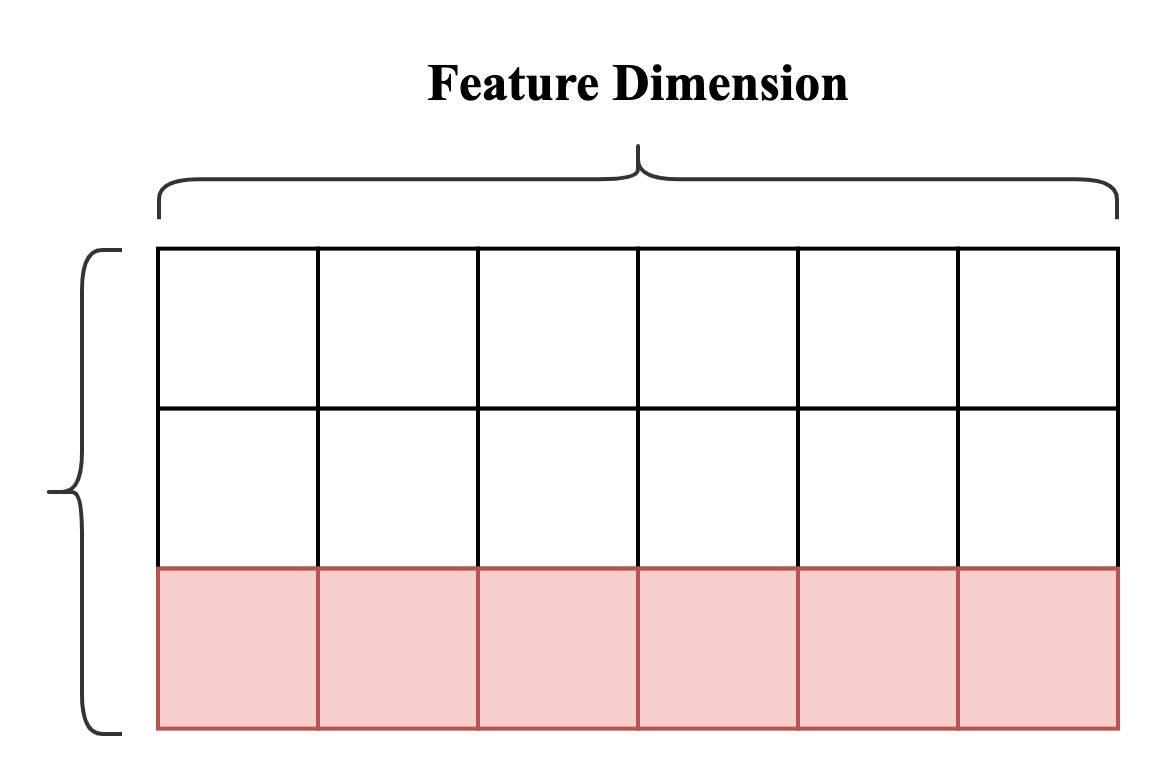}\label{index_info:three}}
\caption{Three domain index information incorporation strategies.}
\label{index_info}
\end{figure*}

We evaluate the algorithms on the EvolCircle and RPlate datasets and the results are reported in Table~\ref{tab:all_res}. The experimental results verify the advantage of our algorithm in exploiting evolving information. We can observe that the improvements induced by incorporating domain index is marginal, which indicates that it cannot properly leverage the evolving pattern of the environment.
 

    \begin{table}[t]
    \caption{Performance of the traditional DG algorithms with domain index information incorporated.}
    \label{tab:all_res}
    \begin{center}
    \adjustbox{max width=\textwidth}{%
    \begin{tabular}{lccccccc}
    \toprule
    \textbf{Strategy}         & \textbf{EvolCircle}  & \textbf{RPlate} & \textbf{Average}              \\
    \midrule
    ERM       &72.7 $\pm$ 1.1              &63.9 $\pm$ 0.9              & 68.3                                  \\
    ERM + One-Dimension        &73.6 $\pm$ 0.6              &64.9 $\pm$ 0.8              & 69.3                                 \\
    ERM + One-Hot              &74.6 $\pm$ 0.3              &64.0 $\pm$ 0.3               & 69.3                                  \\
    ERM + Outer Product        &74.6 $\pm$ 0.4              &65.3 $\pm$ 0.2              & 70.0                                  \\
    DPNets (Ours)              & \textbf{94.2 $\pm$ 0.9}& \textbf{95.0 $\pm$ 0.5}& \textbf{92.2}                                  \\
    \bottomrule
    \end{tabular}}
    \end{center}
    \end{table}

}

\section{Implementation Details}



We implement our algorithm based on \cite{domainbed}. To justify algorithm comparison between baselines and our algorithm, we adopted a random search of 20 trials for the hyper-parameter distribution. For each parameter combination, 5 repeated experiments are conducted. Then, we report the highest average performance for each algorithm-dataset pair. In this way, all parameters are automatically selected without human intervention, making the comparison of experimental results of different algorithms on different data fair and reliable. Almost all backbone and setting are following \cite{domainbed} except the followings. In one singe experiment, the model structure of $f_{\phi}$ and $f_{\psi}$ keeps the same. For EvolCircle and RPlate, we only use one single layer network to make the classifier linear for all algorithms. For other dataset, network are randomly chose based on the random search algorithm. 


\end{document}


\appendix
\section{Directional Domain Generalization Bound}

\subsection{Proof of Theorem 1}

First, we introduce generating divergence bound in DDG TODO:

\begin{theorem}[Generating Divergence Bound TODO]
\label{jsd_decompose}
Let $\mathcal D(x,y)$ and $\mathcal D'(x,y)$ be two distributions over $\mathcal X \times \mathcal Y$, $\mathcal D(y)$ and $\mathcal D'(y)$ be the corresponding marginal distribution of $y$, $\mathcal D(x|y)$ and $\mathcal D'(x|y)$ be the corresponding conditional distribution given $y$, then we can get the following bound,

\begin{align*}
d_{JS}(\mathcal D_t^{g^*}||\mathcal D_t)\leq \frac{1}{m-1}\sum _{i=2}^{m}d_{JS}(\mathcal D_i^{g^*}||\mathcal D_{i})+\lambda
\end{align*}
\end{theorem}

\begin{proof}

According to Definition 1, we have:

\begin{align*}
d_{JS}(\mathcal D^{g^*}_t||\mathcal D_t)\leq d_{JS}(\mathcal D^{g^*}_2||\mathcal D_2)+|d_{JS}(\mathcal D^{g^*}_t||\mathcal D_t)-d_{JS}(\mathcal D_2^{g^*}||\mathcal D_2)|\leq d_{JS}(\mathcal D^{g^*}_2||\mathcal D_2)+\lambda
\end{align*}

Similarly, we have the followings:
\begin{equation*}
d_{JS}(\mathcal D^{g^*}_t||\mathcal D_t)\leq d_{JS}(\mathcal D^{g^*}_i||\mathcal D_i)+|d_{JS}(\mathcal D^{g^*}_t||\mathcal D_t)-d_{JS}(\mathcal D_i^{g^*}||\mathcal D_i)|\leq d_{JS}(\mathcal D^{g^*}_i||\mathcal D_i)+\lambda
\end{equation*}
\begin{equation*}
...
\end{equation*}
\begin{equation*}
d_{JS}(\mathcal D^{g^*}_t||\mathcal D_t)\leq d_{JS}(\mathcal D^{g^*}_m||\mathcal D_m)+|d_{JS}(\mathcal D^{g^*}_t||\mathcal D_t)-d_{JS}(\mathcal D_m^{g^*}||\mathcal D_m)|\leq d_{JS}(\mathcal D^{g^*}_m||\mathcal D_m)+\lambda
\end{equation*}
Sum over all, we have:
\begin{equation*}
(m-1)\cdot d_{JS}(\mathcal D^{g^*}_t||\mathcal D_t)\leq \sum _{i=2}^{m}d_{JS}(\mathcal D_i^{g^*}||\mathcal D_{i+1})+(m-1)\cdot\lambda
\end{equation*}

Then, we have:

\begin{euqtaion*}
$d_{JS}(\mathcal D_t^{g^*}||\mathcal D_t)\leq \frac{1}{m-1}\sum _{i=2}^{m}d_{JS}(\mathcal D_i^{g^*}||\mathcal D_{i})+\lambda$
\end{euqtaion*}
\end{proof}

\begin{theorem}
\label{theoremds}
Let $\{\mathcal D_1,\mathcal D_2,...,\mathcal D_m\}$ be $m$ observed source domains sampled sequentially from an evolving environment $\mathcal{E}$, and $\mathcal{D}_t$ be the next unseen target domain: $\mathcal{D}_t = \mathcal{D}_{m+1}$. Then, if $\mathcal{E}$ is $\lambda$-consistent, we have
\begin{align*}
   R_{\mathcal{D}_t}(h) \le R_{\mathcal{D}_t^{g^*}}(h)+\frac{G}{\sqrt{2(m-1)}}\Bigg(\sqrt{\sum _{i=2}^{m}d_{JS}(\mathcal D_{i}^{g^*}||\mathcal D_{i})}+\sqrt{\lambda}\Bigg).
\end{align*}
\end{theorem}

\begin{proof}

Plug TODO into Lemma TODO, then we have:

\begin{align*}
R_{\mathcal D_{t}}(h)&\leq R_{\mathcal D_t ^{g^*}}(h)+\frac{G}{\sqrt{2}}\sqrt{d_{JS}(\mathcal D_t^{g^*}||\mathcal D_t)}\\
&\leq R_{\mathcal D_t ^{g^*}}(h)+\frac{G}{\sqrt{2}}\sqrt{\frac{1}{m-1}\sum _{i=2}^{m}d_{JS}(\mathcal D_i^{g^*}||\mathcal D_{i})+\lambda}\\
&\leq R_{\mathcal D_t ^{g^*}}(h)+\frac{G}{\sqrt{2(m-1)}}\sqrt{\sum_{i=2}^{m}{d_{JS}(\mathcal D_i^{g^*}||\mathcal D_{i})}}+G\sqrt{\frac{\lambda}{2}}
\end{align*}

\end{proof}

\subsection{Proof of Corollary 1}
We first introducte the Jensen Shannon (J-S) Divergence Decomposition Theorem:
\begin{theorem}[Jensen Shannon (J-S) Divergence Decomposition]
\label{jsd_decompose}
Let $\mathcal D(x,y)$ and $\mathcal D'(x,y)$ be two distributions over $\mathcal X \times \mathcal Y$, $\mathcal D(y)$ and $\mathcal D'(y)$ be the corresponding marginal distribution of $y$, $\mathcal D(x|y)$ and $\mathcal D'(x|y)$ be the corresponding conditional distribution given $y$, then we can get the following bound,

\begin{align*}
d_{JS}(\mathcal D(x,y)||\mathcal D'(x,y))\leq & d_{JS}(\mathcal D(y)||\mathcal D'(y))+\\\mathbb E_{y\sim \mathcal D(y)}d_{JS}(\mathcal D(x|y)||\mathcal D'(x|y))+&\mathbb E_{y\sim \mathcal D'(y)}d_{JS}(\mathcal D(x|y)||\mathcal D'(x|y))
\end{align*}
\end{theorem}

\begin{proof}

According to the definition we can introduce an intermediate distribution 

\begin{equation*}
    \mathcal M(x,y)=\frac{1}{2}(\mathcal D(x,y)+\mathcal D'(x,y))
\end{equation*}

Then we have:
\begin{align*}  
2\cdot d_{JS}(\mathcal D(x,y)||\mathcal D'(x,y))&=d_{KL}(\mathcal D(x,y)||\mathcal M(x,y))+d_{KL}(\mathcal D'(x,y)||\mathcal M(x,y))\\&=d_{KL}(\mathcal D(y)||\mathcal M (y))+\mathbb E_{y\sim \mathcal D(y)} d_{KL}(\mathcal D(x|y)||\mathcal M(x|y))\\&+d_{KL}(\mathcal D'(y)||\mathcal M(y))+\mathbb E_{y\sim \mathcal D'(y)} d_{KL}(\mathcal D(x|y)||\mathcal M(x|y))\\&=2\cdot d_{JS}(\mathcal D(y)||\mathcal D'(y))+\mathbb E_{y\sim \mathcal D(y)} d_{KL}(\mathcal D(x|y)||\mathcal M(x|y))\\&+\mathbb E_{y\sim \mathcal D'(y)} d_{KL}(\mathcal D(x|y)||\mathcal M(x|y))
\end{align*}

To bound the last two terms with JSD, we have:

\begin{align*}
d_{KL}(\mathcal D(x|y)||\mathcal M(x|y))&\leq d_{KL}(\mathcal D(x|y)||\mathcal M(x|y))+d_{KL}(\mathcal D'(x|y)||\mathcal M(x|y))\\&=2\cdot d_{JS}(\mathcal D(x|y)||\mathcal D'(x|y))
\end{align*}

Also,
\begin{align*}
d_{KL}(\mathcal D'(x|y)||\mathcal M(x|y))&\leq d_{KL}(\mathcal D(x|y)||\mathcal M(x|y))+d_{KL}(\mathcal D'(x|y)||\mathcal M(x|y))\\&=2\cdot d_{JS}(\mathcal D(x|y)||\mathcal D'(x|y))
\end{align*}

Now we have:

\begin{align*}
d_{JS}(\mathcal D(x,y)||\mathcal D'(x,y))\leq & d_{JS}(\mathcal D(y)||\mathcal D'(y))+\\\mathbb E_{y\sim \mathcal D(y)}d_{JS}(\mathcal D(x|y)||\mathcal D'(x|y))+&\mathbb E_{y\sim \mathcal D'(y)}d_{JS}(\mathcal D(x|y)||\mathcal D'(x|y))
\end{align*}

which conclude the proof.

\end{proof}

\begin{corollary}
\label{corollay1}
Following the assumptions of Theorem 1, the target risk can be bounded by
\begin{align*}
 &R_{\mathcal D_t}(h) \leq  R_{\mathcal D_t^{g^*}}(h)+\frac{G}{\sqrt{2(m-1)}}\Bigg(\underbrace{\sqrt{\sum _{i=2}^{m}d_{JS}(\mathcal D^{g^*}_{i}(y)||\mathcal D_{i}(y))}}_{\bold{I}}\\
& \hspace{6pt} +\underbrace{\sqrt{\sum _{i=2}^{m}\mathbb E_{y\sim \mathcal D_{i}^{g^*}(y)}d_{JS}(\mathcal D^{g^*}_{i}(x|y)||\mathcal D_{i}(x|y))}}_{\bold{II}}+\underbrace{\sqrt{\sum _{i=2}^{m}\mathbb E_{y\sim \mathcal D_{i}(y)}d_{JS}(\mathcal D^{g^*}_{i}(x|y)||\mathcal D_{i}(x|y))}}_{\bold{III}} + \sqrt{\lambda}\Bigg).
\end{align*}
\end{corollary}

\begin{proof}

\begin{align*}   R_{\mathcal{D}_t}(h) &\le R_{\mathcal D_t ^{g^*}}(h)+\frac{G}{\sqrt{2(m-1)}}\sqrt{\sum_{i=2}^{m}{d_{JS}(\mathcal D_i^{g^*}||\mathcal D_{i})}}+G\sqrt{\frac{\lambda}{2}}\\&\le R_{\mathcal D_t ^{g^*}}(h)+G\sqrt{\frac{\lambda}{2}}+\frac{G}{\sqrt{2(m-1)}}\cdot \\&\sqrt{\sum_{i=2}^{m}{d_{JS}(\mathcal D_i^{g^*}(y)||\mathcal D_i(y))+\mathbb E_{y\sim \mathcal D_i(y)}d_{JS}(\mathcal D_i^{g^*}(x|y)||\mathcal D_i(x|y))+\mathbb E_{y\sim \mathcal D_i^{g^*}(y)}d_{JS}(\mathcal D_i^{g^*}(x|y)||\mathcal D_i(x|y))}}\\
&\\&\le R_{\mathcal D_t^{g^*}}(h)+\frac{G}{\sqrt{2(m-1)}}\Bigg(\underbrace{\sqrt{\sum _{i=2}^{m}d_{JS}(\mathcal D^{g^*}_{i}(y)||\mathcal D_{i}(y))}}_{\bold{I}} \hspace{6pt} +\underbrace{\sqrt{\sum _{i=2}^{m}\mathbb E_{y\sim \mathcal D_{i}^{g^*}(y)}d_{JS}(\mathcal D^{g^*}_{i}(x|y)||\mathcal D_{i}(x|y))}}_{\bold{II}}+\\&\underbrace{\sqrt{\sum _{i=2}^{m}\mathbb E_{y\sim \mathcal D_{i}(y)}d_{JS}(\mathcal D^{g^*}_{i}(x|y)||\mathcal D_{i}(x|y))}}_{\bold{III}}\Bigg)+ G\sqrt{\frac{\lambda}{2}}
\end{align*}

\end{proof}